\documentclass[journal]{IEEEtran}
\usepackage[LGR,T1]{fontenc}
\usepackage[utf8]{inputenc}
\usepackage{colortbl}
\usepackage{array}
\usepackage{float}
\usepackage{multirow}
\usepackage{amsmath}
\usepackage{amsthm}
\usepackage{amssymb}
\usepackage{graphicx}
\usepackage{rotating}
\usepackage[bookmarks=true,bookmarksnumbered=true,bookmarksopen=true,bookmarksopenlevel=1,
 breaklinks=false,pdfborder={0 0 0},pdfborderstyle={},backref=false,colorlinks=false]
 {hyperref}
\hypersetup{pdftitle={Your Title},
 pdfauthor={Your Name},
 pdfpagelayout=OneColumn, pdfnewwindow=true, pdfstartview=XYZ, plainpages=false}

\makeatletter

\DeclareRobustCommand{\greektext}{%
  \fontencoding{LGR}\selectfont\def\encodingdefault{LGR}}
\DeclareRobustCommand{\textgreek}[1]{\leavevmode{\greektext #1}}

\providecommand{\tabularnewline}{\\}
\floatstyle{ruled}
\newfloat{algorithm}{tbp}{loa}
\providecommand{\algorithmname}{Algorithm}
\floatname{algorithm}{\protect\algorithmname}

\usepackage[caption=false,font=footnotesize]{subfig}
\usepackage{algorithm}
\usepackage{algpseudocode}

\makeatother

\theoremstyle{plain}
\newtheorem{thm}{\protect\theoremname}
\theoremstyle{definition}
\newtheorem{defn}[thm]{\protect\definitionname}
\theoremstyle{plain}
\newtheorem{lem}[thm]{\protect\lemmaname}
\providecommand{\definitionname}{Definition}
\providecommand{\lemmaname}{Lemma}
\providecommand{\theoremname}{Theorem}

\begin{document}
\title{Robust Decentralized Personalized Federated Learning via Prediction-Constrained
\\
Neighborhood Collaboration}
\author{{\normalsize Xiao Ma$^{a}$, Hong Shen$^{b}$, Hui Tian$^{c}$, Wenqi
Lyu$^{a}$, Wei Ke$^{a}$}\\
\medskip
{\normalsize$^{a}$Faculty of Applied Sciences, Macao Polytechnic
University, Macao SAR, China}\\
{\normalsize$^{b}$School of Engineering and Technology, Central Queensland
University, Australia}\\
{\normalsize$^{c}$School of Information and Communication Technology,
Griffith University, Australia}}
\maketitle
\begin{abstract}
Decentralized personalized federated learning enables collaborative
model training without centralized coordination while allowing clients
to learn client-specific models under heterogeneous data distributions.
However, robustness in such systems becomes fundamentally challenging
under Byzantine attacks, since honest personalized clients may naturally
produce highly diverse model updates that are difficult to be distinguished
from malicious perturbations. Existing robust decentralized learning
methods typically rely on consensus-oriented assumptions or statistical
agreement among neighboring updates, which become unreliable in decentralized
personalized environments. To address this issue, we propose a robust
decentralized personalized federated learning method R-DPFL, that
enables clients to reduce the impact of Byzantine attacks via robust
neighborhood direction estimation and history-based update trend prediction,
rather than purely aggregating client models as in the existing work.
In R-DPFL, each client first computes the current-round model update
by aggregating the received neighborhood update vectors. It then predicts
what this update should be based on its historical values and local
model changes. Finally, R-DPFL computes the difference between these
two quantities, adaptively clips this difference, and adds it to the
local update. We prove convergence of the learning process through
rigorous analysis and show that honest clients maintain stable personalized
descent dynamics under Byzantine neighbor perturbations without requiring
consensus among neighboring models. Extensive experiments on CIFAR-10
demonstrate that R-DPFL consistently outperforms state-of-the-art
decentralized and personalized federated learning baselines under
heterogeneous and adversarial settings.
\end{abstract}


\section{Introduction}

\IEEEPARstart{F}{ederated} Learning (FL)\cite{mcmahan2017communication}
enables distributed clients to collaboratively train models without
sharing raw data. Most existing FL systems\cite{blanchard2017machine,mcmahan2017communication,yin2018byzantine}
rely on a centralized server to coordinate client updates, which may
introduce a single point of failure, trust concerns, and communication
bottlenecks. To overcome these limitations, decentralized federated
learning (DFL)\cite{nguyen2023balance,jeong2023kdpdfl,sun2022dfedavgm}
removes the central coordinator and allows clients to communicate
directly with their neighbors over a peer-to-peer graph. However,
decentralization also makes learning more vulnerable to unreliable
or malicious neighbors, since each client can only observe local peer
information\cite{fang2024byzantine,raynal2023can,pasquini2023security}.
Meanwhile, real-world federated systems\cite{imteaj2021survey,nguyen2022federated,10262054}
are often highly heterogeneous, where clients may differ in data distributions,
usage patterns, and learning objectives. In such cases, learning one
shared global model is often suboptimal, motivating decentralized
personalized federated learning (DPFL)\cite{dai2022dispfl}, where
each client learns its own personalized model while still benefiting
from neighborhood collaboration.

Existing DPFL methods\cite{dai2022dispfl,liu2024decentralized} have
mainly focused on improving personalization, communication efficiency,
or resource adaptability in benign environments. For example, prior
studies have explored sparse decentralized training\cite{dai2022dispfl},
knowledge distillation\cite{li2021decentralized,jeong2023personalized},
and directed collaboration\cite{liu2024decentralized} to address
data and system heterogeneity. However, these methods typically do
not consider adversarial neighbor behavior during peer-to-peer collaboration.
This limitation becomes critical when personalization and Byzantine
attacks coexist\cite{zhang2025personalized,verbraeken2021bristle}.
In DPFL, even honest clients may naturally produce diverse model updates
because their data distributions, optimization trajectories, and local
objectives differ. As a result, malicious perturbations can be camouflaged
by benign personalized diversity, making them difficult to distinguish
from benign personalized differences\cite{fang2024byzantine,yang2024byzantine}.

This observation reveals a fundamental conflict between robustness
and personalization in decentralized learning. Many Byzantine-robust
decentralized methods are built upon the statistical regularity that
benign updates should remain sufficiently aligned, so that abnormal
messages can be detected or suppressed\cite{yang2024byzantine,fang2024byzantine}.
In contrast, decentralized personalized learning naturally allows
clients to follow different optimization trajectories toward client-specific
solutions\cite{liu2024decentralized,dai2022dispfl}, which increases
benign update diversity and weakens the reliability of similarity-based
anomaly detection. Consequently, robust methods that rely on consensus
may become overly conservative in personalized settings: they may
suppress useful but heterogeneous neighbor information, causing collaboration
to degenerate toward isolated local training. Conversely, overly permissive
mechanisms may allow Byzantine perturbations to propagate through
peer-to-peer communication, degrading the stability and utility of
personalized models\cite{pasquini2023security}. Therefore, the central
challenge is not merely to make decentralized learning robust, but
to achieve robustness without compromising personalization.

This challenge further suggests that the conventional consensus-based
view of decentralized convergence becomes insufficient for DPFL. In
traditional decentralized optimization, convergence is typically characterized
by all honest clients asymptotically approaching a common model or
a shared stationary solution through repeated neighborhood averaging
and information mixing\cite{gabrielli2023survey}. Under this view,
collaborative learning is largely based on consensus: neighboring
updates are expected to remain statistically compatible, and robustness
mechanisms are often designed to identify and suppress updates that
significantly deviate from the majority trend\cite{fang2024byzantine,yang2024byzantine}.

However, such assumptions become fundamentally weaker in decentralized
personalized learning. Due to heterogeneous local data distributions,
device characteristics, user preferences, and application objectives,
different honest clients may naturally evolve toward different local
optima and follow substantially different optimization trajectories
over time\cite{dai2022dispfl,liu2024decentralized}. Consequently,
disagreement among neighboring models is no longer a reliable indicator
of adversarial behavior, but may instead reflect legitimate personalization.
In this setting, forcing strong consensus may undesirably suppress
useful personalized information exchange, reduce model diversity,
and eventually collapse collaborative learning into overly conservative
or even isolated local optimization\cite{marfoq2021federated}.

At the same time, completely removing collaboration is also undesirable\cite{marfoq2021federated},
since neighborhood interaction remains essential for improving generalization,
accelerating convergence, and mitigating local overfitting under limited
local data. Therefore, the goal of robust DPFL should not be to recover
global consensus under Byzantine attacks, but rather to maintain stable
and beneficial personalized optimization in the presence of adversarial
neighbor perturbations. Under this perspective, neighborhood collaboration
should not aim to make neighboring clients converge to the same model.
Instead, it should provide a controlled and reliable correction to
each client's local update, so that useful neighbor information can
improve personalized training without dominating the client's own
objective.

To address this challenge, we propose R-DPFL, a robust decentralized
personalized federated learning method that preserves personalized
local optimization while controlling the influence of unreliable neighbors.
At each communication round, each client first performs local SGD
on its private data and uses the model change after local SGD to construct
a local descent direction. The client then exchanges this direction
with its neighbors and receives possibly corrupted neighbor directions.
Instead of directly averaging these messages or forcing neighboring
models toward agreement, R-DPFL first applies robust aggregation to
obtain an initial estimate of the neighborhood update. It then predicts
the expected neighborhood update from historical changes in the client's
model and historical aggregated neighborhood updates. Finally, the
deviation between the current aggregated update and the predicted
update is adaptively clipped before being added to the local descent
direction. In this way, neighborhood information is used as a bounded
correction to personalized optimization: useful heterogeneous neighbor
information can still contribute to learning, while abrupt or malicious
deviations are prevented from dominating the update.

The main contributions of this work are summarized as follows:
\begin{itemize}
\item We propose a novel method for robust decentralized personalized federated
learning in malicious environments that ensures both preservation
of client personalization and robustness against Byzantine attacks
by computing controlled client model update based on both current-round
neighborhood updates and history-predicted update, rather than forcing
model consensus.
\item We design a personalization-preserved neighborhood direction estimation
mechanism deploying robust aggregation, historical prediction, and
adaptive deviation clipping, that allows each client to suppress the
effect of outliers in the received neighbor update directions while
preserving personalized optimization without enforcing model consensus.
\item We provide a theoretical analysis showing that the proposed update
rule achieves bounded personalized stationarity. Specifically, under
standard smoothness and bounded-error assumptions, each honest client
can maintain stable descent on its own local objective, and the influence
of Byzantine neighbors appears only as a bounded residual term in
the final convergence bound.
\item We conduct experiments under heterogeneous decentralized settings
with multiple Byzantine attacks. The results show that R-DPFL preserves
personalized performance in benign settings, improves robustness under
adversarial neighbor perturbations, and that both the prediction and
clipping components contribute to the final performance.
\end{itemize}

\section{Related Work}

\subsection{Personalized Federated Learning}

Personalized federated learning (PFL) aims to address statistical
heterogeneity across clients by allowing different participants to
learn client-specific models instead of enforcing a single shared
global model. Existing PFL methods mainly differ in how they balance
global collaboration and local personalization. One line of work adopts
partial model personalization, where only a subset of model parameters
are shared across clients. Representative examples include FedPer\cite{arivazhagan2019federated}
and FedBABU\cite{oh2021fedbabu}, which separate shared feature representations
from client-specific classifier layers. Another line of work formulates
federated learning as a multi-task optimization problem. Methods such
as MOCHA\cite{smith2017federated} explicitly model task relations
among clients, while pFedMe\cite{dinh2020personalized} introduces
regularized local optimization to balance global consistency and personalized
adaptation. Knowledge distillation has also been explored to improve
personalization under heterogeneous model architectures. FedMD\cite{li2019fedmd}
and FedDF\cite{lin2020feddf} enable collaborative learning through
prediction-level knowledge transfer instead of direct parameter sharing.
In addition, clustering-based methods such as FedCluster\cite{9377960}
attempt to group clients with similar data distributions to reduce
optimization inconsistency. Although these approaches improve personalization
under heterogeneous data distributions, most existing PFL methods
are developed under centralized federated learning architectures and
assume benign training environments. Consequently, they rely on globally
coordinated aggregation or trusted server-side optimization, making
them difficult to apply directly in adversarial decentralized settings.

\subsection{Decentralized Federated Learning}

Decentralized federated learning (DFL) removes the central server
and allows clients to exchange information directly through peer-to-peer
communication. Compared with centralized FL, DFL improves scalability
and eliminates the dependency on a trusted coordinator, but also introduces
substantial challenges in optimization stability and robustness. Several
recent works have explored personalization in decentralized environments.
DFedAvgM\cite{sun2022dfedavgm} improves communication efficiency
through local momentum and model quantization. DisPFL\cite{dai2022dispfl}
accelerates decentralized personalized learning by introducing personalized
sparse masks and customized local models. KD-PDFL\cite{jeong2023kdpdfl,jeong2023personalized}
incorporates knowledge distillation into decentralized collaboration
to capture statistical relations among clients. From a theoretical
perspective, ARDM\cite{sadiev2022ardm} studies communication and
computation complexity in decentralized personalized optimization.
Despite these advances, existing decentralized personalized learning
methods mainly focus on communication efficiency, model heterogeneity,
or optimization performance under benign environments. Most methods
implicitly assume that neighboring clients exchange reliable information
and do not explicitly consider adversarial peer behavior during decentralized
collaboration.

\subsection{Byzantine-Robust Decentralized Learning}

Byzantine-robust decentralized learning has recently attracted increasing
attention due to the vulnerability of peer-to-peer communication to
malicious participants. Existing robust decentralized methods typically
attempt to suppress abnormal neighbor updates through robust aggregation,
clipping, or consensus-preserving optimization. Representative approaches
include Byzantine-resilient decentralized optimization methods based
on clipping and robust filtering\cite{yang2024byzantine}, as well
as robust decentralized federated learning methods that leverage neighborhood
similarity or robust aggregation to mitigate adversarial influence\cite{fang2024byzantine}.
These methods generally assume that honest clients remain sufficiently
aligned so that malicious deviations can be identified as statistical
outliers. However, such assumptions become substantially weaker in
decentralized personalized learning. In personalized settings, honest
clients may naturally follow different optimization trajectories due
to heterogeneous local objectives and client-specific preferences.
Consequently, disagreement among neighboring updates may no longer
reliably indicate adversarial behavior, which fundamentally challenges
consensus-oriented robustness mechanisms.

Different from existing personalized federated learning methods and
Byzantine-robust decentralized optimization methods, this work studies
robustness and personalization jointly in a decentralized setting.
Existing personalized methods usually assume benign peer collaboration,
while existing Byzantine-robust decentralized methods mainly aim to
protect consensus-oriented learning. In contrast, R-DPFL does not
require neighboring clients to learn the same model. Each client keeps
its own personalized model, constructs a local update direction from
its private data, and uses neighbor information only as a controlled
correction to its local update.

\section{Problem Formulation}

\subsection{Decentralized Personalized Learning System}

We consider a decentralized federated learning system consisting of
$n$ clients, indexed by the set $V=\{1,2,\dots,n\}$. The clients
communicate over a fixed peer-to-peer graph $G=(V,E)$, where an edge
$(i,j)\in E$ indicates that client $i$ can exchange information
with client $j$. For each client $i$, let $N_{i}=\{j\in V:(i,j)\in E\}$
denote its neighbor set. Unlike centralized federated learning, there
is no server to coordinate training or aggregate model updates. Each
client performs local optimization using its private data and communicates
only with its immediate neighbors.

Each client $i$ maintains a personalized local model $w_{i}\in\mathbb{R}^{d}$
and owns a private local dataset $\mathcal{D}_{i}$. The local objective
of client $i$ is defined as 
\begin{equation}
f_{i}(w_{i})=\mathbb{E}_{\xi\sim\mathcal{D}_{i}}[\ell(w_{i};\xi)],
\end{equation}
where $\xi\sim\mathcal{D}_{i}$ denotes that the data sample $\xi$
is (randomly) selected from the local dataset $\mathcal{D}_{i}$,
and $\ell(w_{i};\xi)$ denotes the loss of model $w_{i}$ on $\xi$.
The expectation $\mathbb{E}_{\xi\sim\mathcal{D}_{i}}$ means that
we take the average loss over the local samples of client $i$. Therefore,
$f_{i}(w_{i})$ represents the average local loss of model $w_{i}$.

Due to statistical heterogeneity, the local distributions $\{D_{i}\}_{i\in V}$
may differ significantly across clients. Therefore, different clients
generally do not share the same optimal model. Formally, for two clients
$i\neq j$, we generally have 
\[
\arg\min_{w}f_{i}(w)\neq\arg\min_{w}f_{j}(w_{j}).
\]

Thus, the goal of decentralized personalized learning is not to force
all clients to converge to a single shared model, but to allow each
client to learn a model adapted to its own local objective while still
benefiting from neighborhood collaboration. A common way to describe
such soft collaboration is through the following ideal personalized
formulation: 
\begin{equation}
\min_{\{w_{i}\}_{i\in V}}\sum_{i\in V}f_{i}(w_{i})+\lambda\sum_{(i,j)\in E}\rho(w_{i}-w_{j}),\label{eq:pf}
\end{equation}
where $\lambda>0$ controls the strength of neighborhood collaboration,
and $\rho(\cdot)$ measures the discrepancy between neighboring models.
The first term encourages each client to fit its own local data, while
the second term introduces a soft neighborhood regularization effect.
This Eq.\ref{eq:pf} is used only to motivate soft neighborhood collaboration.
Our method does not directly optimize the pairwise discrepancy term;
instead, it realizes collaboration through robust direction-level
correction.

However, this formulation should not be interpreted as enforcing hard
consensus. In decentralized personalized learning, the discrepancy
term is only intended to encourage useful information exchange among
related clients. The personalized models are still allowed to remain
different due to heterogeneous data distributions and client-specific
objectives.

\subsection{Robust Personalized Collaboration under Byzantine Neighbors}

The above formulation becomes more challenging when some neighbors
are unreliable or Byzantine. In conventional decentralized optimization,
collaboration is often consensus-oriented: neighboring clients are
expected to maintain statistically compatible updates, and large deviations
from the neighborhood trend are usually treated as suspicious. Such
a principle is reasonable when honest clients are expected to optimize
a shared objective or approach a common solution.

In decentralized personalized federated learning, however, this assumption
becomes weaker. Since honest clients may have different local objectives
and personalized optima, their model updates can naturally diverge
even in the absence of attacks. Therefore, a large discrepancy between
neighboring clients should not be automatically regarded as malicious.
It may instead reflect legitimate personalization caused by data heterogeneity,
different optimization trajectories, or client-specific preferences.
This creates a key difficulty: overly conservative robust mechanisms
may suppress useful heterogeneous collaboration, while overly permissive
mechanisms may allow Byzantine perturbations to spread through peer-to-peer
communication.

To capture this setting, we view neighborhood interaction as a controlled
collaborative correction rather than a consensus constraint. At each
communication round $t$, client $i$ updates its personalized model
using two components: a local descent direction $u^{t}_{i}$, computed
from its own data, and a neighborhood collaboration direction $d^{t}_{i}$,
estimated from the messages received from its neighbors. The general
update form is 
\[
w^{t}_{i}=w^{t-1}_{i}-\gamma_{t}\left(u^{t}_{i}+\lambda d^{t}_{i}\right),
\]
where $\gamma_{t}>0$ is the learning rate and $\lambda>0$ controls
the strength of neighborhood collaboration. The local direction $u^{t}_{i}$
drives client $i$ toward its personalized objective, while $d^{t}_{i}$
provides a robust collaborative signal extracted from potentially
corrupted neighbor information.

We consider a Byzantine adversarial setting in which a subset of clients
may send arbitrary messages to their neighbors. Let $u^{t}_{j}$ denote
the update-related message generated by client $j$ at round $t$.
When client $i$ receives information from neighbor $j$, the received
message may be corrupted as 
\[
\tilde{u}^{t}_{j\rightarrow i}=u^{t}_{j}+\delta^{t}_{j\rightarrow i},
\]
where $\delta^{t}_{j\rightarrow i}$ denotes an arbitrary adversarial
perturbation. The perturbation may depend on the current round, local
states, or the adversary's attack strategy. For each honest client
$i$, let $B_{i}\subseteq N_{i}$ denote the set of Byzantine neighbors.
We assume that the Byzantine fraction in each honest client's neighborhood
is bounded: 
\[
\frac{|B_{i}|}{|N_{i}|}\leq\beta<\frac{1}{2}.
\]

Under this threat model, our objective is not to identify every malicious
neighbor or exactly recover every clean neighbor update. Instead,
we aim to construct a neighborhood-level collaboration direction $d^{t}_{i}$
whose influence remains useful under benign heterogeneity and bounded
under Byzantine perturbations. Accordingly, convergence should not
be defined as all honest clients reaching asymptotic consensus. Since
clients may pursue different personalized optima, the desired learning
goal is that each honest client maintains stable descent toward its
own local objective while adversarial neighbor influence only introduces
bounded distortion.

From the theoretical perspective, we aim to establish a bounded personalized
stationarity guarantee. That is, for each honest client $i$, we want
the average expected squared local gradient norm over $T$ iterations
of learning to be bounded by a constant $C$:
\[
\frac{1}{T}\sum^{T-1}_{t=0}\mathbb{E}_{\xi^{t}_{i}}\left[\|\nabla f_{i}(w^{t}_{i})\|^{2}_{2}\right]\le C.
\]
This quantity measures the first-order stationarity of the personalized
local model. Since a stationary point of a differentiable objective
satisfies $\nabla f_{i}(w_{i})=0$, a smaller $\|\nabla f_{i}(w^{t}_{i})\|^{2}_{2}$
indicates that the local model is closer to a stationary point. Here,
$\xi$ $\xi^{t,k}_{i}$ denotes the mini-batch sampled by client $i$
at the $k$-th local SGD step of round $t$, and $\boldsymbol{\xi}^{t}_{i}=(\xi^{t,0}_{i},\ldots,\xi^{t,K-1}_{i})$
collects all mini-batches used by client $i$ in round $t$. The expectation
$\mathbb{E}_{\xi}$ is taken over the mini-batch sampling randomness
along the training trajectory, which determines the random iterate
$w^{t}_{i}$.

\section{Methodology}

\subsection{Framework Overview}

Our robust decentralized personalized federated learning method has
the following framework.

At communication round $t$, each client $i$ maintains a local personalized
model $w^{t-1}_{i}$. Starting from $w^{t-1}_{i}$, the client first
performs multiple local stochastic gradient updates on its private
dataset and constructs a local descent direction $u^{t}_{i}$, which
captures the dominant optimization tendency induced by its own local
objective.

After local optimization, client $i$ exchanges update-related messages
with its neighbors. Since some neighboring clients may be Byzantine,
the received messages may be arbitrarily corrupted and cannot be directly
trusted. To address this issue, each client estimates a robust neighborhood
collaboration direction $d^{t}_{i}$ from the received neighbor messages
together with historical neighborhood dynamics. The estimated direction
is designed to preserve useful collaborative trends while suppressing
abnormal perturbations caused by malicious neighbors.

The personalized model of client $i$ is then updated according to
\[
w^{t}_{i}=w^{t-1}_{i}-\gamma_{t}\left(u^{t}_{i}+\lambda d^{t}_{i}\right),
\]
where $\gamma_{t}>0$ denotes the learning rate and $\lambda>0$ controls
the strength of neighborhood collaboration.

The overall procedure of R-DPFL consists of four main steps at each
communication round:
\begin{enumerate}
\item Local optimization: each client performs multiple local SGD updates
on its private data;
\item Local direction construction: the cumulative local optimization effect
is converted into a local descent direction $u^{t}_{i}$;
\item Robust neighborhood direction estimation: the client estimates a robust
collaborative direction $d^{t}_{i}$ using robust aggregation, historical
prediction, and adaptive deviation clipping;
\item Personalized model update: the personalized model is updated using
both the local descent direction and the robust neighborhood collaboration
direction.
\end{enumerate}
Algorithm \ref{alg:R-DPFL-Training-Procedure} summarizes the complete
training procedure of R-DPFL.

\begin{algorithm}[tbh]
\caption{R-DPFL Training Procedure.\label{alg:R-DPFL-Training-Procedure}}

\begin{algorithmic}

\Require Graph $G=(V,E)$, local step number $K$, learning rates $\eta_t,\gamma_t$, collaboration weight $\lambda$

\State Initialize personalized models $\{w_i^0\}_{i\in V}$

\For{$t=1,2,\dots,T$}

    \Statex \textit{// Each client $i\in V$ performs in parallel}

    \State Set $w_i^{t,0}=w_i^{t-1}$

    \State Perform $K$ steps of local SGD and construct local descent direction
    \[
    u_i^t=\frac{1}{K}\sum_{k=0}^{K-1}g_i(w_i^{t,k};\xi_i^{t,k})
    \]

    \State Exchange local directions with neighboring clients and receive neighbor message set $\widetilde{\mathcal U}_i^t$

    \State Compute robust neighborhood collaboration direction
    \[
    d_i^t = \textsc{RobustDirectionEstimation}(\widetilde{\mathcal U}_i^t)
    \]

    \State Update personalized model:
    \[
    w_i^{t}=w_i^{t-1}-\gamma_t \left(u_i^t+\lambda d_i^t \right)
    \]

\EndFor

\State \textbf{Output:} Personalized models $\{w_i^T\}_{i\in V}$
\end{algorithmic}
\end{algorithm}

\subsection{Local Descent Direction Construction}

At communication round $t$, each client $i$ starts from its current
personalized model $w^{t-1}_{i}$ and performs multiple steps of local
stochastic optimization on its private dataset. Specifically, let
$w^{t,0}_{i}=w^{t-1}_{i},$ for each local step $k=0,1,\dots,K-1$,
client $i$ updates its model according to 
\begin{equation}
w^{t,k+1}_{i}=w^{t,k}_{i}-\eta_{t}g_{i}\left(w^{t,k}_{i};\xi^{t,k}_{i}\right),\label{eq:local_sgd}
\end{equation}
where $\eta_{t}>0$ is the local learning rate, $\xi^{t,k}_{i}$ denotes
a mini-batch sampled from the local dataset of client $i$, and $g_{i}(w^{t,k}_{i};\xi^{t,k}_{i})$
is a stochastic gradient estimator of $\nabla f_{i}(w^{t,k}_{i})$.
After $K$ local steps, client $i$ obtains an intermediate locally
optimized model 
\begin{equation}
\tilde{w}^{t}_{i}=w^{t,K}_{i}.\label{eq:local_model_tilde}
\end{equation}
The intermediate model $\tilde{w}^{t}_{i}$ is only used to construct
the local descent direction, while the actual round-level model is
updated by Eq.\ref{eq:personalized_update}.

Based on the cumulative effect of local optimization, we define the
local descent direction of client $i$ at round $t$ as 
\begin{equation}
u^{t}_{i}=\frac{w^{t-1}_{i}-\tilde{w}^{t}_{i}}{\eta_{t}K}.\label{eq:local_direction}
\end{equation}
Substituting $\eqref{eq:local_sgd}$ into $\eqref{eq:local_direction}$,
we obtain 
\begin{equation}
u^{t}_{i}=\frac{1}{K}\sum^{K-1}_{k=0}g_{i}\left(w^{t,k}_{i};\xi^{t,k}_{i}\right).\label{eq:local_direction_expand}
\end{equation}
Therefore, $u^{t}_{i}$ can be interpreted as the average local stochastic
descent direction accumulated over $K$ optimization steps.

\subsection{Robust Neighborhood Direction Estimation\label{subsec:4_3}}

In decentralized personalized federated learning, neighborhood communication
is the only channel through which clients can exchange collaborative
information. However, in the presence of Byzantine adversaries, the
received neighbor messages may be arbitrarily corrupted and therefore
cannot be directly incorporated into model updates. Moreover, under
personalized data heterogeneity, even benign neighbor directions may
differ substantially from the local direction of a client. Consequently,
the key challenge is not to force agreement among neighbors, but to
extract a stable and informative neighborhood-level collaboration
signal while suppressing abnormal perturbations.

To this end, we construct a robust neighborhood collaboration direction
for each client by combining robust aggregation, temporal prediction,
and adaptive deviation clipping, as shown in Algorithm \ref{alg:Robust-Neighborhood-Direction}.
The resulting direction is designed to preserve benign collaborative
trends while limiting the influence of malicious neighbors to a bounded
range.

\begin{algorithm}[tbh]
\caption{Robust Neighborhood Direction Estimation.\label{alg:Robust-Neighborhood-Direction}}

\begin{algorithmic}

\Require Neighbor message set $\widetilde{\mathcal U}_i^t$, historical model trajectory $\{w_i^{t-1},w_i^{t-2}\}$, historical neighborhood estimates $\{\bar d_i^{t-1},\bar d_i^{t-2}\}$, previous threshold $\tau_i^{t-1}$

\State Compute robust neighborhood aggregation: $\bar d_i^t=\mathrm{RobustAgg}(\widetilde{\mathcal U}_i^t)$
\If{$t\ge2$}

    \State Compute historical variations: $s_i^{t-1}= w_i^{t-1}-w_i^{t-2}$, $y_i^{t-1}=\bar d_i^{t-1}-\bar d_i^{t-2}$

    \State Compute prediction coefficient: $\beta_i^t = \frac{\langle s_i^{t-1},y_i^{t-1}\rangle}{\|s_i^{t-1}\|_2^2+\epsilon}$
	\State Predict neighborhood direction: $\hat d_i^t = \bar d_i^{t-1} + \beta_i^t(w_i^{t-1}-w_i^{t-2})$

\Else

    \State Set $\hat d_i^t=\bar d_i^t$

\EndIf

\State Compute deviation: $\Delta_i^t=\bar d_i^t-\hat d_i^t$

\State Update clipping threshold: $\tau_i^t=\alpha\tau_i^{t-1}+(1-\alpha)c\|\Delta_i^t\|_2$

\State Construct robust neighborhood direction: $d_i^t=\hat d_i^t+\mathrm{clip}(\Delta_i^t;\tau_i^t)$

\State \Return $d_i^t$

\end{algorithmic}
\end{algorithm}

\subsubsection{Neighbor Message Collection}

At communication round $t$, after constructing its local descent
direction $u^{t}_{i}$, each client $i$ transmits this direction
to its neighbors and receives update-related messages from its neighbor
set $N_{i}$. Due to Byzantine attacks, the message sent from neighbor
$j$ to client $i$ may be corrupted. Therefore, the received message
is modeled as 
\begin{equation}
\tilde{u}^{t}_{j\rightarrow i}=u^{t}_{j}+\delta^{t}_{j\rightarrow i},\label{eq:received_message}
\end{equation}
where $\delta^{t}_{j\rightarrow i}$ denotes an arbitrary adversarial
perturbation. Let 
\begin{equation}
\widetilde{\mathcal{U}}^{t}_{i}=\left\{ \tilde{u}^{t}_{j\rightarrow i}\mid j\in\mathcal{N}_{i}\right\} \label{eq:received_set}
\end{equation}
denote the set of all received neighbor directions of client $i$
at round $t$.

Our objective is not to identify which neighbor is malicious, nor
to exactly recover every clean neighbor direction. Instead, we aim
to estimate a neighborhood-level collaboration direction that captures
the dominant benign trend in $\widetilde{\mathcal{U}}^{t}_{i}$ while
remaining robust to outliers and adversarial distortions.

\subsubsection{Robust Neighborhood Aggregation}

As the first step, client $i$ applies a robust aggregation operator
to the received set $\widetilde{\mathcal{U}}^{t}_{i}$ and obtains
a coarse estimate of the current neighborhood collaboration trend:
\begin{equation}
\bar{d}^{t}_{i}=\mathrm{RobustAgg}\left(\widetilde{\mathcal{U}}^{t}_{i}\right),\label{eq:robust_agg}
\end{equation}
where $\mathrm{RobustAgg}(\cdot)$ can be instantiated by standard
robust aggregation rules such as coordinate-wise median, trimmed mean,
or geometric median.

The purpose of this step is to suppress the effect of large instantaneous
outliers in the received neighbor directions. Compared with direct
averaging, robust aggregation provides a more stable estimate of the
dominant neighborhood trend when a bounded fraction of received messages
are corrupted. However, robust aggregation alone is insufficient in
decentralized personalized environments. Since benign neighbor directions
may also vary substantially across clients and over time, a purely
static aggregation rule may still overreact to naturally occurring
heterogeneity or may be misled by temporally coordinated attacks.

\subsubsection{Historical Prediction of Neighborhood Dynamics}

To distinguish benign neighborhood evolution from abrupt abnormal
perturbations, client $i$ first predicts the expected neighborhood
collaboration direction at round $t$ using only historical information
available before receiving the current-round aggregated signal. Let
\begin{equation}
s^{t-1}_{i}=w^{t-1}_{i}-w^{t-2}_{i}\label{eq:s_hist}
\end{equation}
denote the previous-round change in the local model, and let 
\begin{equation}
y^{t-1}_{i}=\bar{d}^{t-1}_{i}-\bar{d}^{t-2}_{i}\label{eq:y_hist}
\end{equation}
denote the previous-round change in the aggregated neighborhood direction.

To capture the local relation between the model trajectory and the
neighborhood evolution, we introduce a scalar prediction coefficient
$\beta^{t}_{i}$ based on a Barzilai-{}-Borwein-style\cite{barzilai1988two}
approximation: 
\begin{equation}
\beta^{t}_{i}=\frac{\left\langle s^{t-1}_{i},y^{t-1}_{i}\right\rangle }{\|s^{t-1}_{i}\|^{2}_{2}+\epsilon},\label{eq:beta_def}
\end{equation}
where $\epsilon>0$ is a small constant for numerical stability. Using
this coefficient, the predicted neighborhood direction at round $t$
is defined as 
\begin{equation}
\hat{d}^{t}_{i}=\bar{d}^{t-1}_{i}+\beta^{t}_{i}\left(w^{t-1}_{i}-w^{t-2}_{i}\right).\label{eq:predicted_direction}
\end{equation}

The intuition behind $\eqref{eq:predicted_direction}$ is that, under
benign optimization dynamics, the neighborhood collaboration trend
should evolve smoothly with the local model trajectory. Therefore,
$\hat{d}^{t}_{i}$ provides a temporally consistent estimate of the
expected neighborhood direction before the current-round neighborhood
observation is incorporated. In contrast, adversarial perturbations
often induce abrupt deviations that are difficult to explain by historical
evolution alone. The coefficient estimates the local sensitivity between
the client model trajectory and the evolution of the neighborhood
collaboration direction, thereby enabling a first-order approximation
of benign neighborhood dynamics.

\subsubsection{Adaptive Deviation Clipping}

After receiving the neighbor messages at round $t$ and computing
the robustly aggregated neighborhood direction $\bar{d}^{t}_{i}$,
client $i$ compares the current observation with the historical prediction.
Specifically, we define the deviation as 
\begin{equation}
\Delta^{t}_{i}=\bar{d}^{t}_{i}-\hat{d}^{t}_{i}.\label{eq:deviation}
\end{equation}
This deviation quantifies how much the currently observed neighborhood
trend differs from the expected benign evolution predicted from history.

Since benign deviations may still arise from data heterogeneity and
non-stationary local optimization, the deviation should not be treated
using a fixed threshold. Instead, we adopt an adaptive clipping threshold
that evolves with historical deviation statistics. Specifically, the
threshold is updated using an exponential moving average: 
\begin{equation}
\tau^{t}_{i}=\alpha\tau^{t-1}_{i}+(1-\alpha)c\|\Delta^{t-1}_{i}\|_{2},\label{eq:tau_update}
\end{equation}
where $\alpha\in[0,1)$ is a momentum coefficient and $c>1$ is a
tolerance factor.

The prediction neighborhood direction is first constructed as 
\[
\tilde{d}^{t}_{i}=\hat{d}^{t}_{i}+Clip\left(\Delta^{t}_{i};\tau^{t}_{i}\right),
\]
where the clipping operator is defined by 
\[
Clip(v;\tau)=v\cdot\min\left\{ 1,\frac{\tau}{\|v\|_{2}}\right\} .
\]
This operation preserves deviations consistent with recent historical
behavior, while clipping large deviations to a bounded limits. To
further control the maximum influence of neighborhood collaboration
on the personalized update, we then apply: 
\[
d^{t}_{i}=Clip\left(\tilde{d}^{t}_{i};G_{d}\right),
\]
where $G_{d}>0$ is the maximum allowed norm of the final neighborhood
collaboration direction. During the warm-up stage, each client sets
$\hat{d}^{t}_{i}=\bar{d}^{t-1}_{i}$ before sufficient historical
information is available. The threshold $\tau^{t}_{i}$ is initialized
as $\tau^{0}_{i}$ and $\Delta^{0}_{i}$ is set to zero.

\subsection{Personalized Descent with Robust Neighborhood Collaboration}

After constructing the local descent direction $u^{t}_{i}$ and the
robust neighborhood collaboration direction $d^{t}_{i}$, client $i$
updates its personalized model by combining these two components in
a unified descent step. Specifically, the update rule at communication
round $t$ is defined as 
\begin{equation}
w^{t}_{i}=w^{t-1}_{i}-\gamma_{t}\left(u^{t}_{i}+\lambda d^{t}_{i}\right),\label{eq:personalized_update}
\end{equation}
where $\gamma_{t}>0$ denotes the learning rate and $\lambda>0$ controls
the strength of neighborhood collaboration.

Unlike consensus-based decentralized optimization methods, the proposed
update mechanism does not directly average neighboring models or enforce
asymptotic agreement among clients. In traditional decentralized learning
schemes, the next iterate is typically obtained by mixing the current
local model with neighbor models or their weighted averages, which
implicitly treats neighborhood consensus as the optimization target.
However, such consensus-oriented updates become less suitable in decentralized
personalized learning, where honest clients may naturally follow heterogeneous
optimization trajectories toward different local optima.

In the proposed method, the local descent direction $u^{t}_{i}$ serves
as the dominant optimization component. It is constructed solely from
the private data of client $i$ and reflects the principal descent
tendency of the local personalized objective. Therefore, $u^{t}_{i}$
drives the model toward a client-specific desirable solution. In contrast,
the neighborhood collaboration direction $d^{t}_{i}$ is not intended
to replace local optimization or enforce agreement with neighboring
models. Instead, it acts as a bounded collaborative correction extracted
from neighborhood interactions, whose role is to introduce useful
collaborative information while suppressing abnormal perturbations
caused by unreliable or Byzantine neighbors.

The collaboration weight $\lambda$ balances personalization and neighborhood
interaction. When $\lambda=0$, the proposed method reduces to fully
independent local training, where each client optimizes only its own
objective without any collaboration. As $\lambda$ increases, neighborhood
information exerts a larger influence on the update process. Nevertheless,
even under larger collaboration weights, the proposed method still
does not require neighboring models to converge toward a common solution.
Instead, collaboration is incorporated in a controlled and reliability-aware
manner through robust aggregation, temporal prediction, and adaptive
deviation clipping.

From an optimization perspective, $\eqref{eq:personalized_update}$
can be interpreted as a personalized descent process with prediction-constrained
robust neighborhood collaboration. The local direction $u^{t}_{i}$
determines the dominant personalized optimization trajectory, while
the neighborhood term $d^{t}_{i}$ introduces only a bounded collaborative
correction whose magnitude is explicitly controlled through the robust
estimation and clipping procedure in Section \ref{subsec:4_3}. Consequently,
Byzantine perturbations can influence the optimization dynamics only
through a bounded neighborhood contribution, while useful heterogeneous
collaborative information can still be preserved.

\section{Theoretical Analysis}

In this section, we analyze the stability and convergence behavior
of the proposed robust decentralized personalized federated learning
method. In our R-DPFL, each honest client performs personalized local
descent while incorporating a bounded neighborhood collaboration term
extracted from potentially corrupted neighbor messages. Therefore,
rather than characterizing convergence through global consensus, we
focus on proving a bounded personalized stationarity guarantee: for
each honest client, the proposed update rule achieves stable descent
on its local objective, while the effect of Byzantine neighbor perturbations
enters only through a bounded collaboration error term.

\subsection{Assumptions}

Following the literation for convergence analysis of FL\cite{li2020convergence,koloskova2020unified,bottou2018optimization}
,we assume that each honest client's local objective satisfies the
following smoothness condition:
\begin{defn}
\label{def:L-smoothness}\cite{gormley2023lecture6}L-smoothness
\end{defn}
A differentiable function $f_{i}:\mathbb{R}^{d}\to\mathbb{R}$ is
said to be $L$-smooth if its gradient is Lipschitz continuous, i.e.,
for any $w,v\in\mathbb{R}^{d}$,
\begin{equation}
\|\nabla f_{i}(w)-\nabla f_{i}(v)\|_{2}\le L\|w-v\|_{2}.
\end{equation}
Equivalently, for any $w,v\in\mathbb{R}^{d}$, 
\begin{equation}
f_{i}(v)\le f_{i}(w)+\langle\nabla f_{i}(w),v-w\rangle+\frac{L}{2}\|v-w\|^{2}_{2}.
\end{equation}

\subsubsection*{Assumption 1. \label{Ass:Unbiased} \cite{li2020federated,mcmahan2017communication,stich2019local}Unbiased
gradients with bounded norm.}

For each honest client $i$, the stochastic gradient computed over
mini-batch sample $\xi$ drawn from $i$'s local dataset, $g_{i}(w;\xi)$,
is an unbiased estimator of $i$'s local gradient: 
\[
\mathbb{E}_{\xi}\left[g_{i}(w;\xi)\right]=\nabla f_{i}(w).
\]
Moreover, there exists a constant $G_{g}>0$ such that 
\[
\mathbb{E}_{\xi}\left\Vert g_{i}(w;\xi)\right\Vert ^{2}_{2}\le G^{2}_{g}.
\]
This condition also implies bounded stochastic gradient variance,
since 
\[
\mathbb{E}_{\xi}\left\Vert g_{i}(w;\xi)-\nabla f_{i}(w)\right\Vert ^{2}_{2}\le G^{2}_{g}.
\]

\subsection{Proof Sketch}

For an honest client $i$, we view one training round as consisting
of two components, a personalized local descent component $u^{t}_{i}$
and a corrected neighborhood collaboration component $d^{t}_{i}$.
The proof first analyzes these two components separately and then
combines them in the one-step progress of the local objective $f_{i}$.
\begin{enumerate}
\item We first show that the local descent direction $u^{t}_{i}$ remains
aligned with the personalized objective of client $i$. Under standard
stochastic gradient assumptions, $u^{t}_{i}$ can be decomposed into
the true local gradient $\nabla f_{i}(w^{t-1}_{i})$, a bounded local
drift bias, and a bounded stochastic noise term.
\item We then show that the corrected neighborhood collaboration direction
$d^{t}_{i}$ has bounded influence. Since Byzantine messages may be
arbitrary, we do not bound each malicious message directly. Instead,
the proposed robust aggregation, historical prediction, and adaptive
clipping mechanism transforms unreliable neighborhood information
into a bounded residual.
\item Finally, we show that the personalized descent effect dominates this
bounded collaboration residual in the one-step progress of $f_{i}$.
After accumulating the descent inequalities over communication rounds,
Byzantine neighbor perturbations appear only as a bounded residual
term in the final bound, yielding a bounded personalized guarantee
for honest client $i$.
\end{enumerate}

\subsection{Bounded Local Descent Direction}

At the end of round $t$, after receiving neighborhood information
and performing $K$-step local updates with mini-batch sampling $\boldsymbol{\xi}^{t}_{i}$,
client $i$ obtains the updated model $w^{t}_{i}$, which is then
used as the initialization for the next round $t+1$. We first show
that the local descent direction is a controlled approximation of
the true local gradient, with bounded local drift and stochastic variance.
\begin{lem}
\label{lem:2}(Bounded local descent direction) Let $\boldsymbol{\xi}^{t}_{i}=(\xi^{t,0}_{i},\ldots,\xi^{t,K-1}_{i})$
denote all mini-batches sampled by client $i$ during $K$ local SGD
steps in round $t$. Conditional on the beginning round model $w^{t-1}_{i}$,
the expectation is taken over the local mini-batch samples $\boldsymbol{\xi}^{t}_{i}$.
Suppose that Assumption 1 holds and that $f_{i}$ is $L$-smooth.
For each honest client $i$ and round $t$, the local descent direction
$u^{t}_{i}=\frac{1}{K}\sum^{K-1}_{k=0}g_{i}(w^{t,k}_{i};\xi^{t,k}_{i})$
satisfies 
\[
\mathbb{E}_{\boldsymbol{\xi}^{t}_{i}}\left[u^{t}_{i}\mid w^{t-1}_{i}\right]=\nabla f_{i}(w^{t-1}_{i})+b^{t}_{i},\qquad\|b^{t}_{i}\|_{2}\le B\eta_{t}K,
\]
and 
\[
\mathbb{E}_{\boldsymbol{\xi}^{t}_{i}}\left[\left\Vert u^{t}_{i}-\mathbb{E}_{\boldsymbol{\xi}^{t}_{i}}[u^{t}_{i}\mid w^{t-1}_{i}]\right\Vert ^{2}_{2}\Bigm|w^{t-1}_{i}\right]\le\sigma^{2}_{u},
\]
where $B>0$ and $\sigma^{2}_{u}\ge0$ are constants, and $b^{t}_{i}=\frac{1}{K}\sum^{K-1}_{k=0}\mathbb{E}_{\boldsymbol{\xi}^{t}_{i}}\left[\nabla f_{i}(w^{t,k}_{i})-\nabla f_{i}(w^{t-1}_{i})\mid w^{t-1}_{i}\right]$
is the local drift bias induced by the $K$ local SGD steps. Here,
$w^{t,k}_{i}$ denotes the local model of client $i$ after $k$ local
SGD steps in round $t$, starting from the beginning-round model $w^{t-1}_{i}$,
i.e., $w^{t,0}_{i}=w^{t-1}_{i},w^{t,k}_{i}=w^{t-1}_{i}-\eta\sum^{k-1}_{r=0}g_{i}(w^{t,r}_{i};\xi^{t,r}_{i}),\quad k=1,\ldots,K.$
\end{lem}
\begin{IEEEproof}
By the definition of the local descent direction, we have $u^{t}_{i}=\frac{1}{K}\sum^{K-1}_{k=0}g_{i}(w^{t,k}_{i};\xi^{t,k}_{i}).$Taking
the conditional expectation over all mini-batches $\boldsymbol{\xi}^{t}_{i}$
sampled in round $t$, conditioned on the model $w^{t-1}_{i}$, gives
\[
\mathbb{E}_{\boldsymbol{\xi}^{t}_{i}}\left[u^{t}_{i}\mid w^{t-1}_{i}\right]=\frac{1}{K}\sum^{K-1}_{k=0}\mathbb{E}_{\boldsymbol{\xi}^{t}_{i}}\left[g_{i}(w^{t,k}_{i};\xi^{t,k}_{i})\mid w^{t-1}_{i}\right].
\]
Notice that $w^{t,k}_{i}$ depends on $w^{t-1}_{i}$ and the previous
mini-batches $\xi^{t,0}_{i},\ldots,\xi^{t,k-1}_{i}$. Using Assumption
1, we obtain
\[
\begin{aligned} & \mathbb{E}_{\boldsymbol{\xi}^{t}_{i}}\left[g_{i}(w^{t,k}_{i};\xi^{t,k}_{i})\mid w^{t-1}_{i}\right]\\
 & \quad=\mathbb{E}_{\boldsymbol{\xi}^{t}_{i}}\left[\mathbb{E}_{\xi^{t,k}_{i}}\left[g_{i}(w^{t,k}_{i};\xi^{t,k}_{i})\mid w^{t,k}_{i}\right]\Bigm|w^{t-1}_{i}\right]\\
 & \quad=\mathbb{E}_{\boldsymbol{\xi}^{t}_{i}}\left[\nabla f_{i}(w^{t,k}_{i})\mid w^{t-1}_{i}\right].
\end{aligned}
\]
Therefore, 
\[
\mathbb{E}_{\boldsymbol{\xi}^{t}_{i}}\left[u^{t}_{i}\mid w^{t-1}_{i}\right]=\frac{1}{K}\sum^{K-1}_{k=0}\mathbb{E}_{\boldsymbol{\xi}^{t}_{i}}\left[\nabla f_{i}(w^{t,k}_{i})\mid w^{t-1}_{i}\right].
\]
Define the local drift bias as $b^{t}_{i}=\mathbb{E}_{\boldsymbol{\xi}^{t}_{i}}\left[u^{t}_{i}\mid w^{t-1}_{i}\right]-\nabla f_{i}(w^{t-1}_{i}).$
Since $\nabla f_{i}(w^{t-1}_{i})$ is fixed when conditioning on $w^{t-1}_{i}$,
the bias can be written as 
\[
\begin{aligned}b^{t}_{i} & =\mathbb{E}_{\boldsymbol{\xi}^{t}_{i}}[u^{t}_{i}\mid w^{t-1}_{i}]-\nabla f_{i}(w^{t-1}_{i})\\
 & =\frac{1}{K}\sum^{K-1}_{k=0}\mathbb{E}_{\boldsymbol{\xi}^{t}_{i}}[\nabla f_{i}(w^{t,k}_{i})\mid w^{t-1}_{i}]-\nabla f_{i}(w^{t-1}_{i})\\
 & =\frac{1}{K}\sum^{K-1}_{k=0}\left(\mathbb{E}_{\boldsymbol{\xi}^{t}_{i}}[\nabla f_{i}(w^{t,k}_{i})\mid w^{t-1}_{i}]-\nabla f_{i}(w^{t-1}_{i})\right)\\
 & =\frac{1}{K}\sum^{K-1}_{k=0}\mathbb{E}_{\boldsymbol{\xi}^{t}_{i}}\left[\nabla f_{i}(w^{t,k}_{i})-\nabla f_{i}(w^{t-1}_{i})\mid w^{t-1}_{i}\right].
\end{aligned}
\]
Consequently, 
\[
\mathbb{E}_{\boldsymbol{\xi}^{t}_{i}}\left[u^{t}_{i}\mid w^{t-1}_{i}\right]=\nabla f_{i}(w^{t-1}_{i})+b^{t}_{i}.
\]

Here, $b^{t}_{i}$ represents the local drift bias between the expected
$K$-step local descent direction and the gradient at the model $w^{t-1}_{i}$,
while $\|b^{t}_{i}\|_{2}$ quantifies the size of this drift. By Jensen's
inequality, the triangle inequality, and the $L$-smoothness of $f_{i}$,
we have 
\[
\begin{aligned}\|b^{t}_{i}\|_{2} & =\left\Vert \frac{1}{K}\sum^{K-1}_{k=0}\mathbb{E}_{\boldsymbol{\xi}^{t}_{i}}\left[\nabla f_{i}(w^{t,k}_{i})-\nabla f_{i}(w^{t-1}_{i})\mid w^{t-1}_{i}\right]\right\Vert _{2}\\
 & \le\frac{1}{K}\sum^{K-1}_{k=0}\left\Vert \mathbb{E}_{\boldsymbol{\xi}^{t}_{i}}\left[\nabla f_{i}(w^{t,k}_{i})-\nabla f_{i}(w^{t-1}_{i})\mid w^{t-1}_{i}\right]\right\Vert _{2}\\
 & \le\frac{1}{K}\sum^{K-1}_{k=0}\mathbb{E}_{\boldsymbol{\xi}^{t}_{i}}\left[\left\Vert \nabla f_{i}(w^{t,k}_{i})-\nabla f_{i}(w^{t-1}_{i})\right\Vert _{2}\mid w^{t-1}_{i}\right]\\
 & \le\frac{L}{K}\sum^{K-1}_{k=0}\mathbb{E}_{\boldsymbol{\xi}^{t}_{i}}\left[\|w^{t,k}_{i}-w^{t-1}_{i}\|_{2}\mid w^{t-1}_{i}\right].
\end{aligned}
\]

From the local SGD, 
\[
w^{t,k}_{i}=w^{t-1}_{i}-\eta\sum^{k-1}_{r=0}g_{i}(w^{t,r}_{i};\xi^{t,r}_{i}),
\]
and hence 
\[
\begin{aligned} & \mathbb{E}_{\boldsymbol{\xi}^{t}_{i}}\left[\|w^{t,k}_{i}-w^{t-1}_{i}\|_{2}\mid w^{t-1}_{i}\right]\\
 & \quad\le\eta_{t}\sum^{k-1}_{r=0}\mathbb{E}_{\boldsymbol{\xi}^{t}_{i}}\left[\|g_{i}(w^{t,r}_{i};\xi^{t,r}_{i})\|_{2}\mid w^{t-1}_{i}\right].
\end{aligned}
\]
By the Cauchy-{}-Schwarz inequality, 
\[
\begin{aligned} & \mathbb{E}_{\boldsymbol{\xi}^{t}_{i}}\left[\|g_{i}(w^{t,r}_{i};\xi^{t,r}_{i})\|_{2}\mid w^{t-1}_{i}\right]\\
 & \quad\le\left(\mathbb{E}_{\boldsymbol{\xi}^{t}_{i}}\left[\|g_{i}(w^{t,r}_{i};\xi^{t,r}_{i})\|^{2}_{2}\mid w^{t-1}_{i}\right]\right)^{1/2}\le G_{g},
\end{aligned}
\]
where the last inequality follows from Assumption 1. It follows that
\[
\mathbb{E}_{\boldsymbol{\xi}^{t}_{i}}\left[\|w^{t,k}_{i}-w^{t-1}_{i}\|_{2}\mid w^{t-1}_{i}\right]\le\eta kG_{g}.
\]
Substituting this bound into the bias estimate gives 
\[
\begin{aligned}\|b^{t}_{i}\|_{2} & \le\frac{L\eta G_{g}}{K}\sum^{K-1}_{k=0}k\\
 & =\frac{L\eta_{t}G_{g}(K-1)}{2}\le B\eta K,
\end{aligned}
\]
where one may choose 
\[
B=\frac{LG_{g}}{2}.
\]
It remains to bound the stochastic fluctuation of $u^{t}_{i}$.

Let $\mu^{t}_{i}=\mathbb{E}_{\boldsymbol{\xi}^{t}_{i}}\left[u^{t}_{i}\mid w^{t-1}_{i}\right]$,
using the conditional variance identity, we obtain
\[
\begin{aligned} & \mathbb{E}_{\boldsymbol{\xi}^{t}_{i}}\left[\|u^{t}_{i}-\mu^{t}_{i}\|^{2}_{2}\mid w^{t-1}_{i}\right]\\
 & \quad=\mathbb{E}_{\boldsymbol{\xi}^{t}_{i}}\left[\|u^{t}_{i}\|^{2}_{2}\mid w^{t-1}_{i}\right]-\|\mu^{t}_{i}\|^{2}_{2}\\
 & \quad\le\mathbb{E}_{\boldsymbol{\xi}^{t}_{i}}\left[\|u^{t}_{i}\|^{2}_{2}\mid w^{t-1}_{i}\right].
\end{aligned}
\]
By Jensen's inequality,
\[
\begin{aligned} & \mathbb{E}_{\boldsymbol{\xi}^{t}_{i}}\left[\|u^{t}_{i}\|^{2}_{2}\mid w^{t-1}_{i}\right]\\
 & =\mathbb{E}_{\boldsymbol{\xi}^{t}_{i}}\left[\left\Vert \frac{1}{K}\sum^{K-1}_{k=0}g_{i}(w^{t,k}_{i};\xi^{t,k}_{i})\right\Vert ^{2}_{2}\Bigm|w^{t-1}_{i}\right]\\
 & \le\frac{1}{K}\sum^{K-1}_{k=0}\mathbb{E}_{\boldsymbol{\xi}^{t}_{i}}\left[\|g_{i}(w^{t,k}_{i};\xi^{t,k}_{i})\|^{2}_{2}\mid w^{t-1}_{i}\right]\\
 & \le G^{2}_{g}.
\end{aligned}
\]
Therefore,
\[
\mathbb{E}_{\boldsymbol{\xi}^{t}_{i}}\left[\left\Vert u^{t}_{i}-\mathbb{E}_{\boldsymbol{\xi}^{t}_{i}}[u^{t}_{i}\mid w^{t-1}_{i}]\right\Vert ^{2}_{2}\Bigm|w^{t-1}_{i}\right]\le\sigma^{2}_{u},
\]
where one may take $\sigma^{2}_{u}=G^{2}_{g}$.

This completes the proof.
\end{IEEEproof}

\subsection{Bounded Neighborhood Collaboration Direction}

We then show that historical prediction and adaptive clipping transform
the corrected neighborhood collaboration direction into a bounded
residual term in the personalized update.
\begin{lem}
\label{lem:3}(Bounded neighborhood collaboration direction) Let the
prediction neighborhood direction be constructed as $\tilde{d}^{t}_{i}=\hat{d}^{t}_{i}+Clip\left(\bar{d}^{t}_{i}-\hat{d}^{t}_{i};\tau^{t}_{i}\right),$
and the final neighborhood collaboration direction be $d^{t}_{i}=Clip\left(\tilde{d}^{t}_{i};G_{d}\right),$
where the clipping operator is defined as $Clip(v;\tau)=v\cdot\min\left\{ 1,\frac{\tau}{\|v\|_{2}}\right\} .$
Then the final neighborhood collaboration direction satisfies
\[
\|d^{t}_{i}\|_{2}\le G_{d}.
\]
\end{lem}
\begin{IEEEproof}
By the definition of the final neighborhood collaboration direction,
we have 
\[
d^{t}_{i}=Clip\left(\tilde{d}^{t}_{i};G_{d}\right).
\]
From the definition of the clipping operator, for any vector $v$
and radius $\tau>0$, 
\[
\left\Vert Clip(v;\tau)\right\Vert _{2}=\left\Vert v\cdot\min\left\{ 1,\frac{\tau}{\|v\|_{2}}\right\} \right\Vert _{2}.
\]
Therefore,
\[
\left\Vert Clip(v;\tau)\right\Vert _{2}=\|v\|_{2}\min\left\{ 1,\frac{\tau}{\|v\|_{2}}\right\} =\min\left\{ \|v\|_{2},\tau\right\} \le\tau.
\]

Taking $v=\tilde{d}^{t}_{i}$ and $\tau=G_{d}$, we obtain 
\[
\|d^{t}_{i}\|_{2}=\left\Vert Clip\left(\tilde{d}^{t}_{i};G_{d}\right)\right\Vert _{2}\le G_{d}.
\]

This completes the proof.
\end{IEEEproof}

\subsection{Stable Personalized Descent}

We now establish that the proposed update rule achieves a bounded
personalized stationarity guarantee for each honest client. The key
idea is to treat the corrected neighborhood collaboration direction
as a bounded residual term in the personalized descent process.
\begin{thm}
(Stable personalized descent)  For an    honest client $i$ with initial
local model $w^{0}_{i}$, if  $w_{i}$ is updated by $w^{t}_{i}=w^{t-1}_{i}-\gamma\left(u^{t}_{i}+\lambda d^{t}_{i}\right)$
according to our R-DPFL Training Procedure Algorithm \ref{alg:R-DPFL-Training-Procedure}
during the $t$-th round of learning,  where $\gamma>0$ is a constant
step size satisfying $\gamma\le\frac{1}{8L}$, and finally stabilized
after $T$ rounds of learning. Then, under Assumption 1 and the bounded
neighborhood direction condition $\|d^{t}_{i}\|_{2}\le G_{d}$, the
average personalized descent during these $T$rounds  is bounded to
a constant, i.e.,  the following inequality holds

\begin{eqnarray}
\frac{1}{T}\sum^{T}_{t=1}\left\Vert \nabla f_{i}(w^{t}_{i})\right\Vert ^{2}_{2} & \le & \frac{4\left(f_{i}(w^{0}_{i})-f^{\star}_{i}\right)}{\gamma T}+\label{eq:18}\\
 &  & 5B^{2}\eta^{2}K^{2}+\frac{1}{2}\sigma^{2}_{u}+\frac{9}{2}\lambda^{2}G^{2}_{d}\nonumber 
\end{eqnarray}
where, $f^{\star}_{i}=f_{i}(w^{T}_{i})$, $B$, $\eta$, $K$ and
$\sigma^{2}_{u}$ are constants defined in Lemma \ref{lem:2}.
\end{thm}
\begin{IEEEproof}
We decompose the one-step descent update into the local direction,
the neighborhood collaboration direction, and the quadratic term.
By bounding these terms, we obtain a descent inequality, which is
then averaged over past samples and telescoped over $T$ rounds.

Let $g^{t-1}_{i}=\nabla f_{i}(w^{t-1}_{i})$, and $\mathcal{F}_{t}$
be the history available before the model update of round $t$, including
all samples generated before round $t$. In round $t$, client $i$
samples mini-batches $\boldsymbol{\xi}^{t}_{i}=(\xi^{t,0}_{i},\ldots,\xi^{t,K-1}_{i})$,
receives neighborhood information, and computes $u^{t}_{i}$ and $d^{t}_{i}$.
Conditioned on $\mathcal{F}_{t}$, the past sampled mini-batches are
fixed, while the current mini-batches $\boldsymbol{\xi}^{t}_{i}$
are freshly sampled from $\mathcal{D}_{i}$ and remain random.

The updated model is then given by 
\[
w^{t}_{i}=w^{t-1}_{i}-\gamma\left(u^{t}_{i}+\lambda d^{t}_{i}\right).
\]
Conditioned on $\mathcal{F}_{t}$, both $w^{t-1}_{i}$ and $g^{t-1}_{i}$
are fixed quantities, while $u^{t}_{i}$, $d^{t}_{i}$, and $w^{t}_{i}$
depend on the samples generated during round $t$\cite{li2020convergence,koloskova2020unified,bottou2018optimization}.

Since $f_{i}$ is $L$-smooth, we have 
\[
f_{i}(w^{t}_{i})\le f_{i}(w^{t-1}_{i})+\left\langle g^{t-1}_{i},w^{t}_{i}-w^{t-1}_{i}\right\rangle +\frac{L}{2}\left\Vert w^{t}_{i}-w^{t-1}_{i}\right\Vert ^{2}_{2}.
\]
Using 
\[
w^{t}_{i}=w^{t-1}_{i}-\gamma\left(u^{t}_{i}+\lambda d^{t}_{i}\right).
\]
we obtain 
\[
f_{i}(w^{t}_{i})\le f_{i}(w^{t-1}_{i})-\gamma\left\langle g^{t-1}_{i},u^{t}_{i}+\lambda d^{t}_{i}\right\rangle +\frac{L\gamma^{2}}{2}\left\Vert u^{t}_{i}+\lambda d^{t}_{i}\right\Vert ^{2}_{2}.
\]
Taking expectation conditional on $\mathcal{F}_{t}$ in both sides
gives 
\begin{equation}
\begin{aligned}\mathbb{E}\left[f_{i}(w^{t}_{i})\mid\mathcal{F}_{t}\right] & \le\mathbb{E}\left[f_{i}(w^{t-1}_{i})\mid\mathcal{F}_{t}\right]\\
 & \quad-\gamma\mathbb{E}\left[\left\langle g^{t-1}_{i},u^{t}_{i}\right\rangle \mid\mathcal{F}_{t}\right]\\
 & \quad-\gamma\lambda\mathbb{E}\left[\left\langle g^{t-1}_{i},d^{t}_{i}\right\rangle \mid\mathcal{F}_{t}\right]\\
 & \quad+\frac{L\gamma^{2}}{2}\mathbb{E}\left[\left\Vert u^{t}_{i}+\lambda d^{t}_{i}\right\Vert ^{2}_{2}\mid\mathcal{F}_{t}\right]\\
 & =\mathbb{E}\left[f_{i}(w^{t-1}_{i})\mid\mathcal{F}_{t}\right]+A+B+C,
\end{aligned}
\label{eq:final_all}
\end{equation}

For $A$, we first bound the local linear term $-\mathbb{E}\left[\left\langle g^{t-1}_{i},u^{t}_{i}\right\rangle \mid\mathcal{F}_{t}\right]$.
Since $g^{t-1}_{i}$ is fixed conditioned on $\mathcal{F}_{t}$, we
have 
\[
\begin{aligned} & -\mathbb{E}\left[\left\langle g^{t-1}_{i},u^{t}_{i}\right\rangle \mid\mathcal{F}_{t}\right]\\
 & \quad=-\left\langle g^{t-1}_{i},\mathbb{E}\left[u^{t}_{i}\mid\mathcal{F}_{t}\right]\right\rangle .
\end{aligned}
\]
By Lemma \ref{lem:2}, the local descent direction satisfies 
\[
\mathbb{E}\left[u^{t}_{i}\mid\mathcal{F}_{t}\right]=g^{t-1}_{i}+b^{t}_{i},\qquad\|b^{t}_{i}\|_{2}\le B\eta K.
\]
Therefore, 
\[
\begin{aligned} & -\mathbb{E}\left[\left\langle g^{t-1}_{i},u^{t}_{i}\right\rangle \mid\mathcal{F}_{t}\right]\\
 & \quad=-\left\langle g^{t-1}_{i},g^{t-1}_{i}+b^{t}_{i}\right\rangle \\
 & \quad=-\|g^{t-1}_{i}\|^{2}_{2}-\left\langle g^{t-1}_{i},b^{t}_{i}\right\rangle .
\end{aligned}
\]
Using Young's inequality, 
\[
-\left\langle g^{t-1}_{i},b^{t}_{i}\right\rangle \le\frac{1}{4}\|g^{t-1}_{i}\|^{2}_{2}+\|b^{t}_{i}\|^{2}_{2}.
\]
Hence, 
\begin{equation}
-\mathbb{E}\left[\left\langle g^{t-1}_{i},u^{t}_{i}\right\rangle \mid\mathcal{F}_{t}\right]\le-\frac{3}{4}\|g^{t-1}_{i}\|^{2}_{2}+\|b^{t}_{i}\|^{2}_{2}.\label{eq:final_A}
\end{equation}

For $B$, we bound the neighborhood linear term $-\gamma\lambda\mathbb{E}\left[\left\langle g^{t-1}_{i},d^{t}_{i}\right\rangle \mid\mathcal{F}_{t}\right]$.
Since $g^{t-1}_{i}$ is fixed conditioned on $\mathcal{F}_{t}$, by
Young's inequality, for each realization of $d^{t}_{i}$, we have
\begin{align*}
 & -\lambda\left\langle g^{t-1}_{i},d^{t}_{i}\right\rangle \\
 & \le\lambda\left|\left\langle g^{t-1}_{i},d^{t}_{i}\right\rangle \right|\\
 & \le\lambda\|g^{t-1}_{i}\|_{2}\|d^{t}_{i}\|_{2}\\
 & \le\frac{1}{4}\|g^{t-1}_{i}\|^{2}_{2}+\lambda^{2}\|d^{t}_{i}\|^{2}_{2}
\end{align*}
Taking conditional expectation conditioned on $\mathcal{F}_{t}$ gives,
\[
\begin{aligned} & -\lambda\mathbb{E}\left[\left\langle g^{t-1}_{i},d^{t}_{i}\right\rangle \mid\mathcal{F}_{t}\right]\\
 & \quad\le\frac{1}{4}\|g^{t-1}_{i}\|^{2}_{2}+\lambda^{2}\mathbb{E}\left[\|d^{t}_{i}\|^{2}_{2}\mid\mathcal{F}_{t}\right].
\end{aligned}
\]
By Lemma \ref{lem:3}, $\|d^{t}_{i}\|^{2}_{2}\le G^{2}_{d}$. Therefore,
\begin{equation}
-\lambda\mathbb{E}\left[\left\langle g^{t-1}_{i},d^{t}_{i}\right\rangle \mid\mathcal{F}_{t}\right]\le\frac{1}{4}\|g^{t-1}_{i}\|^{2}_{2}+\lambda^{2}G^{2}_{d}.\label{eq:final_B}
\end{equation}

For $C$, we bound the quadratic term $\mathbb{E}\left[\left\Vert u^{t}_{i}+\lambda d^{t}_{i}\right\Vert ^{2}_{2}\mid\mathcal{F}_{t}\right]$.
Using $\|a+b\|^{2}_{2}\le2\|a\|^{2}_{2}+2\|b\|^{2}_{2},$ we have
\[
\begin{aligned} & \mathbb{E}\left[\left\Vert u^{t}_{i}+\lambda d^{t}_{i}\right\Vert ^{2}_{2}\mid\mathcal{F}_{t}\right]\\
 & \quad\le2\mathbb{E}\left[\|u^{t}_{i}\|^{2}_{2}\mid\mathcal{F}_{t}\right]+2\lambda^{2}\mathbb{E}\left[\|d^{t}_{i}\|^{2}_{2}\mid\mathcal{F}_{t}\right].
\end{aligned}
\]
By the conditional variance decomposition and Lemma \ref{lem:2},
\[
\begin{aligned} & \mathbb{E}\left[\|u^{t}_{i}\|^{2}_{2}\mid\mathcal{F}_{t}\right]\\
 & \quad=\left\Vert \mathbb{E}\left[u^{t}_{i}\mid\mathcal{F}_{t}\right]\right\Vert ^{2}_{2}+\mathbb{E}\left[\left\Vert u^{t}_{i}-\mathbb{E}\left[u^{t}_{i}\mid\mathcal{F}_{t}\right]\right\Vert ^{2}_{2}\mid\mathcal{F}_{t}\right]\\
 & \quad\le\|g^{t-1}_{i}+b^{t}_{i}\|^{2}_{2}+\sigma^{2}_{u}.
\end{aligned}
\]
Using $\|a+b\|^{2}_{2}\le2\|a\|^{2}_{2}+2\|b\|^{2}_{2}$, we further
obtain 
\[
\mathbb{E}\left[\|u^{t}_{i}\|^{2}_{2}\mid\mathcal{F}_{t}\right]\le2\|g^{t-1}_{i}\|^{2}_{2}+2\|b^{t}_{i}\|^{2}_{2}+\sigma^{2}_{u}.
\]
By Lemma \ref{lem:3}, $\|d^{t}_{i}\|^{2}_{2}\le G^{2}_{d}.$ Therefore,
\[
\mathbb{E}\left[\|d^{t}_{i}\|^{2}_{2}\mid\mathcal{F}_{t}\right]\le G^{2}_{d}.
\]
Combining the above bounds yields 
\begin{equation}
\begin{aligned} & \mathbb{E}\left[\left\Vert u^{t}_{i}+\lambda d^{t}_{i}\right\Vert ^{2}_{2}\mid\mathcal{F}_{t}\right]\\
 & \quad\le4\|g^{t-1}_{i}\|^{2}_{2}+4\|b^{t}_{i}\|^{2}_{2}+2\sigma^{2}_{u}+2\lambda^{2}G^{2}_{d}.
\end{aligned}
\label{eq:final_C}
\end{equation}

Substituting   $A$, $B$ and $C$ in (\ref{eq:final_all}) with $(\ref{eq:final_A}),(\ref{eq:final_B})$
and (\ref{eq:final_C}) yields 
\[
\begin{aligned} & \mathbb{E}\left[f_{i}(w^{t}_{i})\mid\mathcal{F}_{t}\right]\\
 & \le\mathbb{E}\left[f_{i}(w^{t-1}_{i})\mid\mathcal{F}_{t}\right]-\left(\frac{\gamma}{2}-2L\gamma^{2}\right)\|g^{t-1}_{i}\|^{2}_{2}+\left(\gamma+2L\gamma^{2}\right)\|b^{t}_{i}\|^{2}_{2}\\
 & \quad+\left(\gamma+L\gamma^{2}\right)\lambda^{2}G^{2}_{d}+L\gamma^{2}\sigma^{2}_{u}.
\end{aligned}
\]
Since $\gamma\le\frac{1}{8L}$, we have 
\[
\frac{\gamma}{2}-2L\gamma^{2}\ge\frac{\gamma}{4},\qquad\gamma+2L\gamma^{2}\le\frac{5\gamma}{4},
\]
and 
\[
\gamma+L\gamma^{2}\le\frac{9\gamma}{8},\qquad L\gamma^{2}\le\frac{\gamma}{8}.
\]
Because $\|b^{t}_{i}\|_{2}\le B\eta K$ by Lemma \ref{lem:2}, combining
the above we have
\begin{eqnarray*}
\mathbb{E}\left[f_{i}(w^{t}_{i})\mid\mathcal{F}_{t}\right] & \le & \mathbb{E}\left[f_{i}(w^{t-1}_{i})\mid\mathcal{F}_{t}\right]-\frac{\gamma}{4}\|g^{t-1}_{i}\|^{2}_{2}\\
 &  & +\frac{5\gamma}{4}B^{2}\eta^{2}K^{2}+\frac{9\gamma}{8}\lambda^{2}G^{2}_{d}+\frac{\gamma}{8}\sigma^{2}_{u}
\end{eqnarray*}

Rearranging terms yields
\begin{align}
\frac{\gamma}{4}\|g^{t-1}_{i}\|^{2}_{2}\le & \mathbb{E}\left[f_{i}(w^{t-1}_{i})\mid\mathcal{F}_{t}\right]-\mathbb{E}\left[f_{i}(w^{t}_{i})\mid\mathcal{F}_{t}\right]\label{eq: all-rearrange}\\
 & +\frac{5\gamma}{4}B^{2}\eta^{2}K^{2}+\frac{9\gamma}{8}\lambda^{2}G^{2}_{d}+\frac{\gamma}{8}\sigma^{2}_{u}\nonumber 
\end{align}

Since $w^{t-1}_{i}$ is determined by the information of $\mathcal{F}_{t-1}$
in round $t-1$, it is $\mathcal{F}_{t-1}$-measurable and hence$f_{i}(w^{t-1}_{i})$
is $\mathcal{F}_{t-1}$-measurable because $f_{i}$ is deterministic.
Because the added information in $\mathcal{F}_{t}$ in round $t$
over $\mathcal{F}_{t-1}$ has no impact on $w^{t-1}_{i}$ generated
in the preceding round, $w^{t-1}_{i}$ and $f_{i}(w^{t-1}_{i})$ are
also $\mathcal{F}_{t}$-measurable. Hence, by the measurability property
of conditional probability in stochastic optimization, we have:
\[
\mathbb{E}\left[f_{i}(w^{t-1}_{i})\mid\mathcal{F}_{t}\right]=f_{i}(w^{t-1}_{i})
\]
\[
\mathbb{E}\left[f_{i}(w^{t}_{i})\mid\mathcal{F}_{t}\right]=f_{i}(w^{t}_{i})
\]

Replacing the expectation terms of  (\ref{eq: all-rearrange}) with
the above and computing the average of both sides from $t=1$ to $T$
yields:
\begin{align*}
\frac{\gamma}{4}\sum^{T}_{t=1}\|g^{t-1}_{i}\|^{2}_{2} & \le\sum^{T}_{t=1}\Bigg(f_{i}(w^{t-1}_{i})-f_{i}(w^{t}_{i})\Bigg)+\\
 & T\left(\frac{5\gamma}{4}B^{2}\eta^{2}K^{2}+\frac{9\gamma}{8}\lambda^{2}G^{2}_{d}+\frac{\gamma}{8}\sigma^{2}_{u}\right)\\
= & f(w^{0}_{i})-f(w^{T}_{i})+\\
 & T\left(\frac{5\gamma}{4}B^{2}\eta^{2}K^{2}+\frac{9\gamma}{8}\lambda^{2}G^{2}_{d}+\frac{\gamma}{8}\sigma^{2}_{u}\right)\\
= & f_{i}(w^{0}_{i})-f^{\star}_{i}+\\
 & T\left(\frac{5\gamma}{4}B^{2}\eta^{2}K^{2}+\frac{9\gamma}{8}\lambda^{2}G^{2}_{d}+\frac{\gamma}{8}\sigma^{2}_{u}\right)
\end{align*}

The inequality of the theorem follows directly by dividing both sides by
$\frac{\gamma T}{4}$ and replacing $g^{t-1}_{i}$ with $\nabla f_{i}(w^{t-1}_{i})$.
\end{IEEEproof}

\section{Experiments}

\subsection{Experimental Objectives}

The experiments are designed to evaluate the proposed method from
three aspects.
\begin{enumerate}
\item we examine whether the proposed defense can preserve competitive personalized
learning performance in the absence of attacks.
\item we investigate whether the proposed method can effectively improve
robustness against Byzantine neighbors under different attack strengths.
\item we study whether the performance gains of the proposed method indeed
come from its key design components.
\end{enumerate}

\subsection{Datasets and Models}

We conduct experiments on widely used image classification benchmarks,
CIFAR-10, to evaluate the robustness of the proposed method under
decentralized federated learning with Byzantine neighbors.The datasets
consist of $32\times32$ RGB natural images. CIFAR-10 contains 10
object categories and serves as a standard benchmark for robust federated
learning. For CIFAR-10, we employ a lightweight convolutional neural
network as the local model. The adopted architecture follows a standard
convolutional design, consisting of multiple convolutional layers
for hierarchical feature extraction and fully connected layers for
classification.

\subsection{Experimental Setup}

We simulate a decentralized personalized federated learning environment
with 10 clients. All clients participate in every communication round,
and communication is performed over a fully connected topology, where
each client exchanges update-related messages with all other clients.
To model statistical heterogeneity, local datasets are partitioned
using a Dirichlet distribution with concentration parameter $\alpha=0.3$.
Each experiment is run for 500 communication rounds. For local optimization,
each client performs 1 local epoch of stochastic gradient descent
per round with batch size 64, local learning rate 0.01, momentum 0,
and weight decay 0.

For the proposed R-DPFL method, the global update coefficient is set
to 1.0. The neighborhood collaboration weight $\lambda$ is varied
according to the purpose of each experiment. In clean and sensitivity
studies, we examine $\lambda\in\{0,0.005,0.01,0.05\}$. In the main
robustness and ablation experiments, we use $\lambda=0.01$ as the
default setting. The warm-up stage lasts for 5 communication rounds.
The exponential moving average coefficient for adaptive threshold
estimation is set to 0.9, the tolerance factor is set to 1.5, and
the initial threshold is set to 1.0. For numerical stability, we use
$\epsilon=10^{-8}$ in both the Barzilai-{}-Borwein-style prediction
coefficient and the clipping operation, and clamp the prediction coefficient
to the range $[-10,10]$.

\subsection{Attack Settings and Evaluation Metrics}

\subsubsection*{Attack Models.}

To evaluate robustness under adversarial environments, we consider
several representative Byzantine attack models commonly used in federated
and decentralized learning. The Gaussian noise attack perturbs transmitted
directions with random noise and serves as a basic non-strategic baseline\cite{blanchard2017machine}.
The sign-flipping attack reverses the direction of model updates to
maximally oppose the optimization objective and is widely used to
simulate strong adversarial behavior\cite{yin2018byzantine}. We also
consider a scaling attack, where malicious clients amplify their transmitted
directions by a large factor in order to dominate neighborhood aggregation\cite{baruch2019little}.
In addition, we include two omniscient attack variants, namely the
min-max attack and the min-sum attack, in which Byzantine clients
craft adversarial messages using the honest updates observed in the
same round\cite{fang2020local}. To further assess defense-aware adversaries,
we additionally consider an adaptive attack. Instead of sending an
obviously abnormal direction, each Byzantine client crafts a poisoned
message around the predicted neighborhood trend and then applies a
small perturbation in the opposite direction of the honest neighborhood
mean. This makes the corrupted message appear more plausible while
still biasing the aggregated neighborhood signal. All attacks are
implemented in a decentralized message-poisoning manner, where Byzantine
clients independently corrupt the directions sent to their neighbors.
Unless otherwise specified, we set 3 out of 10 clients as Byzantine
attackers, corresponding to a Byzantine ratio of 30\%.

\subsubsection*{Evaluation protocol and metrics.}

Our primary objective is to evaluate the robustness of the personalized
models held by honest participants. Since Byzantine attackers may
maintain arbitrary or corrupted model parameters, including them in
the evaluation would obscure the true performance of the system. Therefore,
we report the Average Test Accuracy of Honest Clients. Let $\mathcal{H}$
denote the set of honest clients. The metric is defined as:
\begin{equation}
\text{Acc}^{t}_{\text{honest}}=\frac{1}{|\mathcal{H}|}\sum_{i\in\mathcal{H}}\text{Acc}(w^{t}_{i};\mathcal{D}_{test}),
\end{equation}
where $\text{Acc}(w^{t}_{i};\mathcal{D}_{test})$ represents the classification
accuracy of honest client $i$'s model on the held-out test set. This
metric rigorously measures whether the collaborative defense mechanism
successfully protects the utility of honest clients against adversarial
influence.

\subsection{Main Results}

\subsubsection{Clean Performance}

\begin{table}[tbh]
\caption{Comparison of average honest personalized accuracy under different
values of $\lambda$ in the clean setting.\label{tab:clean}}

\centering{}%
\begin{tabular}{|c|c|c|c|c|}
\hline 
\multirow{1}{*}{Aggregation} & $\lambda=0$ & $\lambda=0.005$ & $\lambda=0.01$ & $\lambda=0.05$\tabularnewline
\hline 
Mean & 0.8185 & 0.8167 & 0.8206 & 0.8245\tabularnewline
\hline 
Median & 0.8185 & 0.8218 & 0.8202 & 0.8231\tabularnewline
\hline 
\end{tabular}
\end{table}

\begin{figure*}[tbh]
\centering{}\includegraphics[scale=0.45]{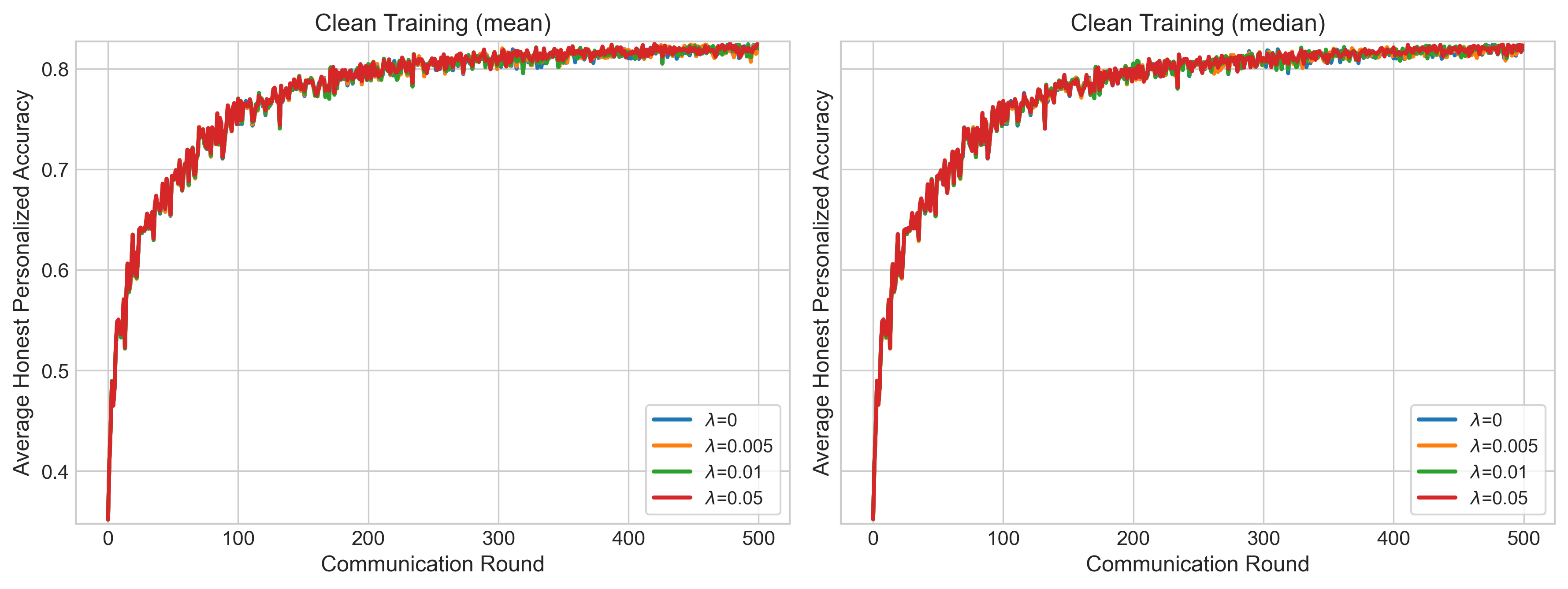}

\caption{Comparison of average honest personalized performance under different
values of $\lambda$ in the clean setting\label{fig:clean_per}}
\end{figure*}

We first evaluate the proposed method under the clean setting, where
no malicious clients are present, in order to examine whether the
introduced neighborhood correction mechanism harms normal personalized
training. Fig.\ref{fig:clean_per} shows the clean training curves
under mean-based and median-based aggregation with different values
of $\lambda$. In our method, mean-based aggregation and median-based
aggregation are not treated as two separate methods. Instead, they
are two alternative implementations of the neighborhood aggregation
component used to obtain the aggregated neighborhood direction. All
settings exhibit stable convergence throughout the training process,
indicating that the proposed method remains well-behaved under clean
conditions. The average honest personalized accuracy achieved by different
$\lambda$ values is highly comparable under both aggregation rules.
In particular, the final performance gap among $\lambda=0$, $\lambda=0.005$,
$\lambda=0.01$, and $\lambda=0.05$ is relatively small, which suggests
that introducing the neighborhood correction term does not noticeably
degrade the quality of local personalized models. A moderate neighborhood
regularization strength, such as $\lambda=0.05$, shows a slight advantage
in the middle and later stages of training, although the gain remains
limited overall.

\subsubsection{Robustness under Byzantine Attacks}

\begin{table*}[tbh]
\centering{}\caption{Comparison of final honest personalized performance under different
Byzantine attack settings.\label{tab:all}}

\centering{}%
\begin{tabular}{|c|c|c|c|c|c|c|c|c|c|c|c|c|c|}
\hline 
\multirow{2}{*}{Datasets} & \multirow{2}{*}{Defense Method} & \multicolumn{2}{c|}{sign\_flip} & \multicolumn{2}{c|}{noise} & \multicolumn{2}{c|}{scaling} & \multicolumn{2}{c|}{min\_max} & \multicolumn{2}{c|}{min\_sum} & \multicolumn{2}{c|}{adaptive}\tabularnewline
\cline{3-14}
 &  & Honest & All & Honest & All & Honest & All & Honest & All & Honest & All & Honest & All\tabularnewline
\hline 
\multirow{5}{*}{Cifar10} & Mean & 0.8172 & 0.8195 & 0.8215 & 0.8215 & 0.8192 & 0.8204 & 0.8178 & 0.8198 & 0.8236 & 0.8231 & 0.8167 & 0.8194\tabularnewline
\cline{2-14}
 & Median & 0.8172 & 0.8195 & 0.8215 & 0.8215 & 0.8192 & 0.8204 & 0.8178 & 0.8198 & 0.8236 & 0.8231 & 0.8167 & 0.8194\tabularnewline
\cline{2-14}
 & Trimmed-Mean & 0.8172 & 0.8195 & 0.8215 & 0.8215 & 0.8192 & 0.8204 & 0.8178 & 0.8198 & 0.8236 & 0.8231 & 0.8167 & 0.8194\tabularnewline
\cline{2-14}
 & DisPFL & 0.0861 & 0.1001 & 0.1184 & 0.0972 & 0.0861 & 0.1001 & 0.0861 & 0.1001 & 0.0861 & 0.1001 & 0.0861 & 0.1001\tabularnewline
\cline{2-14}
 & Our & 0.8172 & 0.8195 & 0.8215 & 0.8215 & 0.8192 & 0.8204 & 0.8178 & 0.8198 & 0.8236 & 0.8231 & 0.8167 & 0.8194\tabularnewline
\hline 
\end{tabular}
\end{table*}

\begin{figure}[tbh]
\caption{Training curves of average honest personalized accuracy under different
Byzantine attacks.\label{fig:ALL}}

\centering{}\includegraphics[scale=0.4]{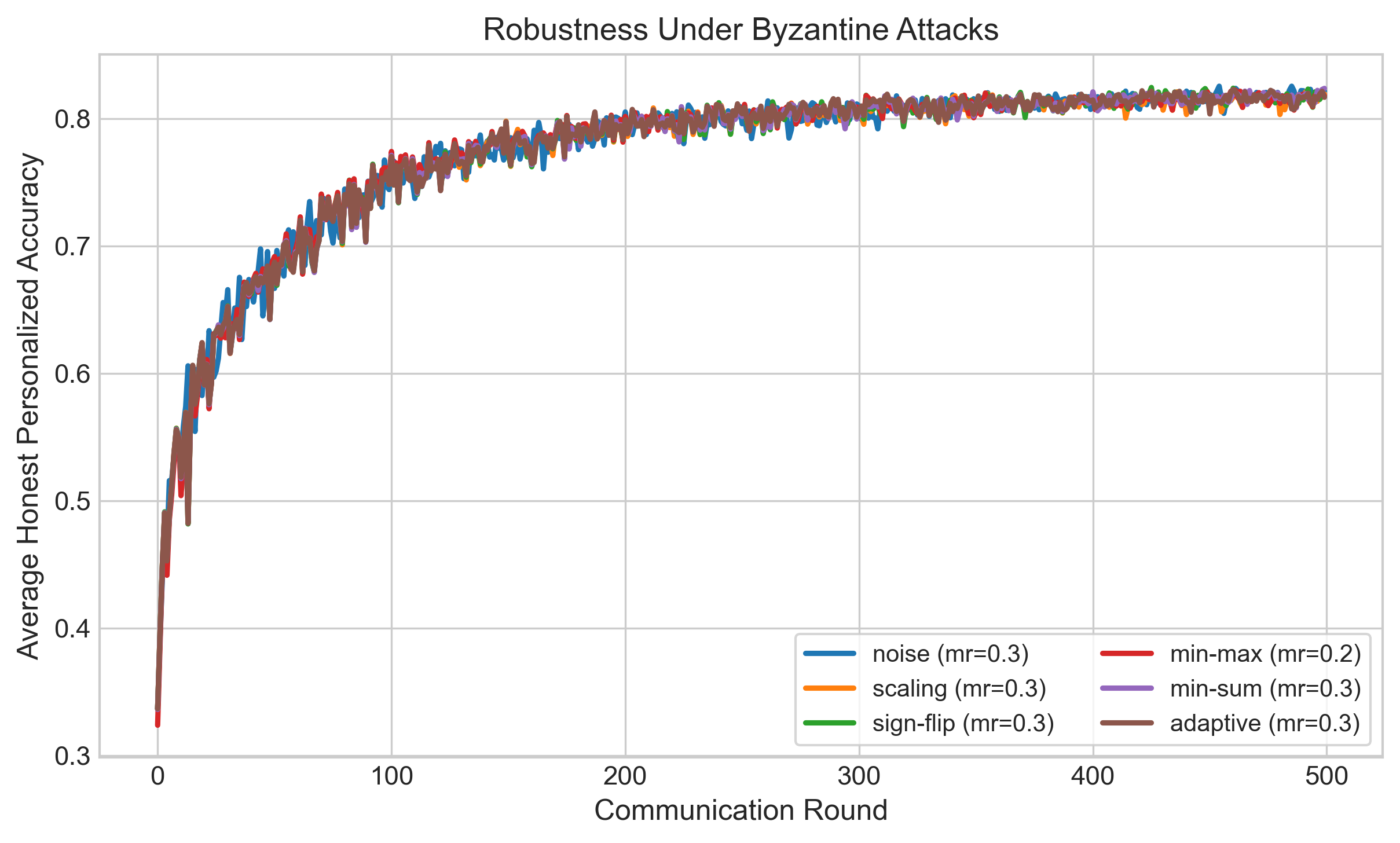}
\end{figure}

\begin{figure}[tbh]
\caption{Comparison of final honest personalized accuracy under different Byzantine
attacks, with the clean reference shown as a dashed line..\label{fig:ALL-1}}

\centering{}\includegraphics[scale=0.35]{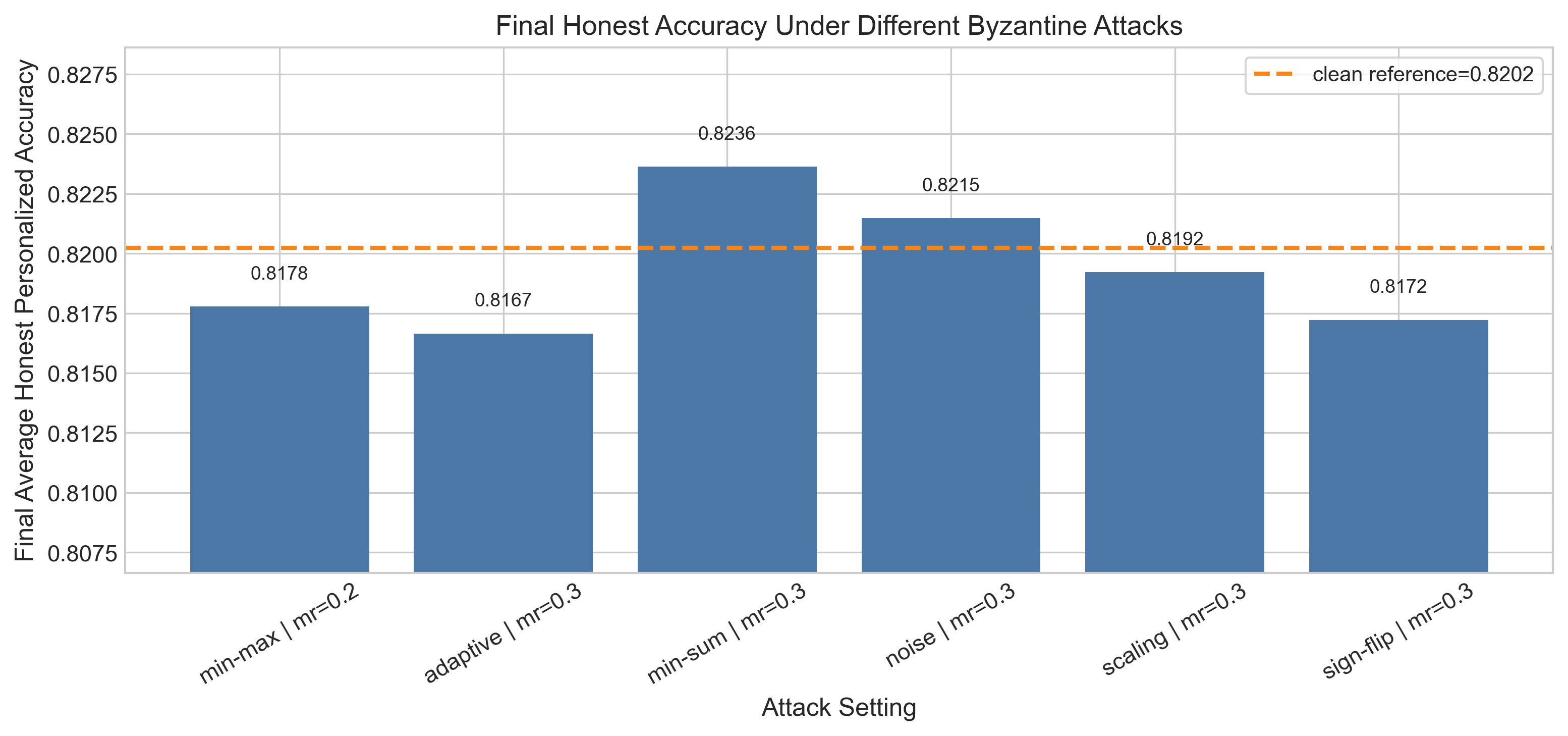}
\end{figure}

We next evaluate the proposed method under multiple Byzantine attack
settings in order to examine whether it can maintain stable personalized
performance when neighborhood information is corrupted by malicious
clients. Specifically, we consider three standard perturbation-based
attacks, namely sign-flip, noise, and scaling, two stronger optimization-based
attacks, namely min-max and min-sum, and one moderate defense-aware
adaptive attack.

Table $\ref{tab:all}$ reports the final personalized performance
under different attack settings, where we present both the average
accuracy over honest clients and the average accuracy over all clients.
Several observations can be made. First, the proposed method maintains
consistently high personalized accuracy under all considered attacks,
indicating strong robustness against a range of Byzantine perturbations.
Second, the gap between the honest-client accuracy and the overall
average accuracy remains small across all attack settings, suggesting
that the learned personalized models of honest clients are not severely
degraded by the presence of malicious neighbors. Third, among the
considered attacks, the adaptive, sign-flip, and min-max attacks lead
to relatively larger performance drops, while the degradation remains
limited overall. By contrast, the results under noise and min-sum
are very close to, or slightly above, the clean reference. We attribute
these small differences to the stochasticity of decentralized training
rather than a genuine performance gain caused by attacks.

Fig.$\ref{fig:ALL}$ further shows the training curves of average
honest personalized accuracy under different Byzantine attacks. It
can be observed that all attack settings still exhibit stable convergence
throughout the communication rounds, and no attack causes catastrophic
collapse or severe training instability.

To provide a more direct comparison, Fig. $\ref{fig:ALL-1}$ presents
the final average honest personalized accuracy under different attack
settings, with the clean result shown as a dashed reference line.
We observe that the final performance under all attacks remains close
to the clean reference, which further confirms the robustness of the
proposed method. Although adaptive attack produces one of the largest
performance drops, its impact is still moderate, indicating that the
proposed defense remains effective even against a defense-aware adversary.

\subsection{Ablation and Sensitivity Analysis}

\begin{table}[tbh]
\caption{Ablation results of key design components under different Byzantine
attacks.\label{tab:Ablation}}

\centering{}%
\begin{tabular}{|c|c|c|c|}
\hline 
attack\_type & variant & Honest & All\tabularnewline
\hline 
\hline 
\multirow{3}{*}{adaptive} & Full Method & 0.8167 & 0.8194\tabularnewline
\cline{2-4}
 & w/o Clipping & 0.8196 & 0.8199\tabularnewline
\cline{2-4}
 & w/o Prediction & 0.8196 & 0.8199\tabularnewline
\hline 
\multirow{3}{*}{min\_max} & Full Method & 0.8178 & 0.8198\tabularnewline
\cline{2-4}
 & w/o Clipping & 0.8162 & 0.8179\tabularnewline
\cline{2-4}
 & w/o Prediction & 0.8162 & 0.8179\tabularnewline
\hline 
\multirow{3}{*}{min\_sum} & Full Method & 0.8236 & 0.8231\tabularnewline
\cline{2-4}
 & w/o Clipping & 0.8153 & 0.818\tabularnewline
\cline{2-4}
 & w/o Prediction & 0.8153 & 0.818\tabularnewline
\hline 
\end{tabular}
\end{table}

\begin{figure}[tbh]
\caption{Comparison of final honest personalized accuracy for different ablated
variants under representative Byzantine attacks.\label{fig:ablation}}

\centering{}\includegraphics[scale=0.4]{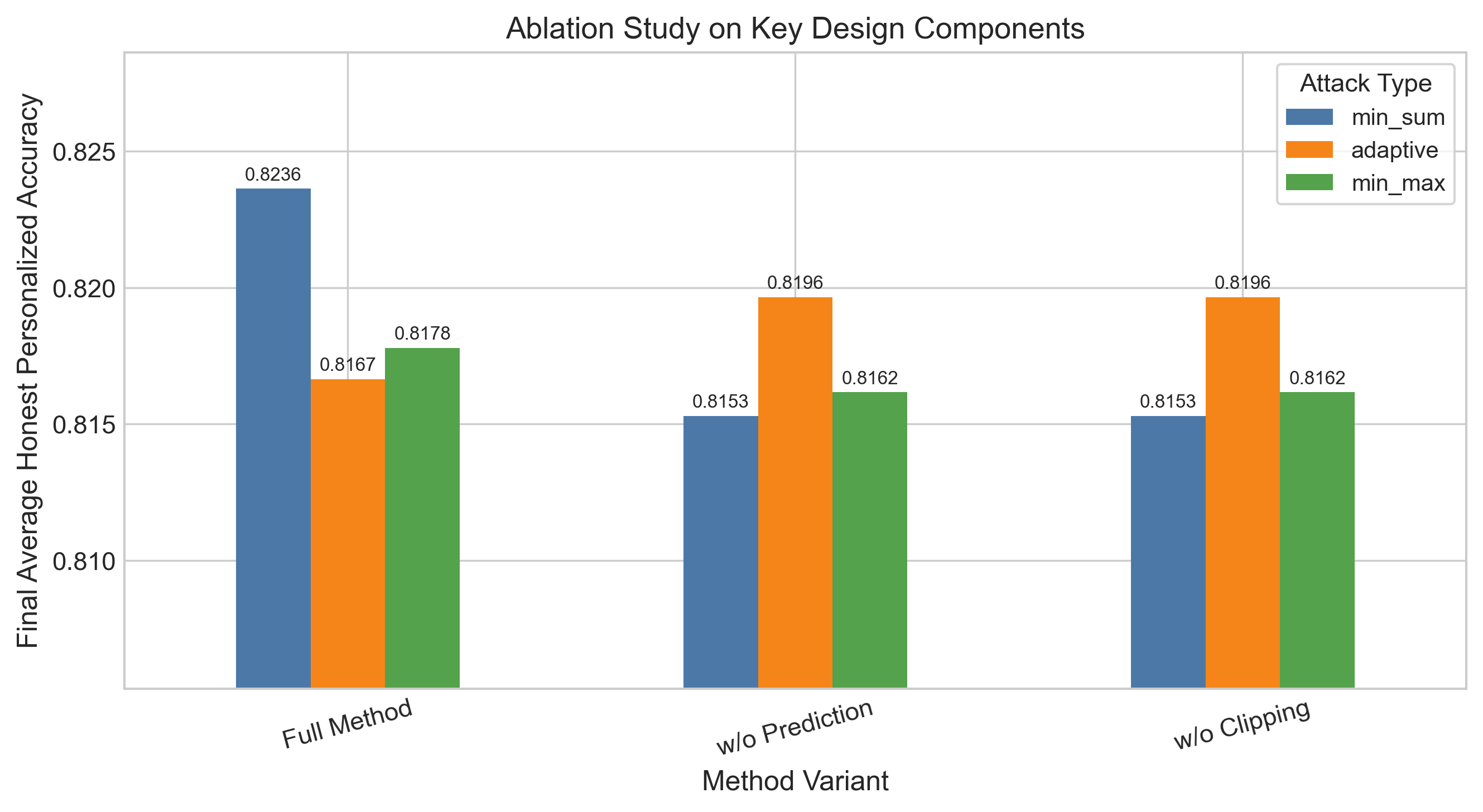}
\end{figure}

\begin{figure*}[tbh]
\centering
\caption{Training curves of average honest personalized accuracy for different
ablated variants under representative Byzantine attacks.\label{fig:ablation-1}}

\centering{}\includegraphics[scale=0.45]{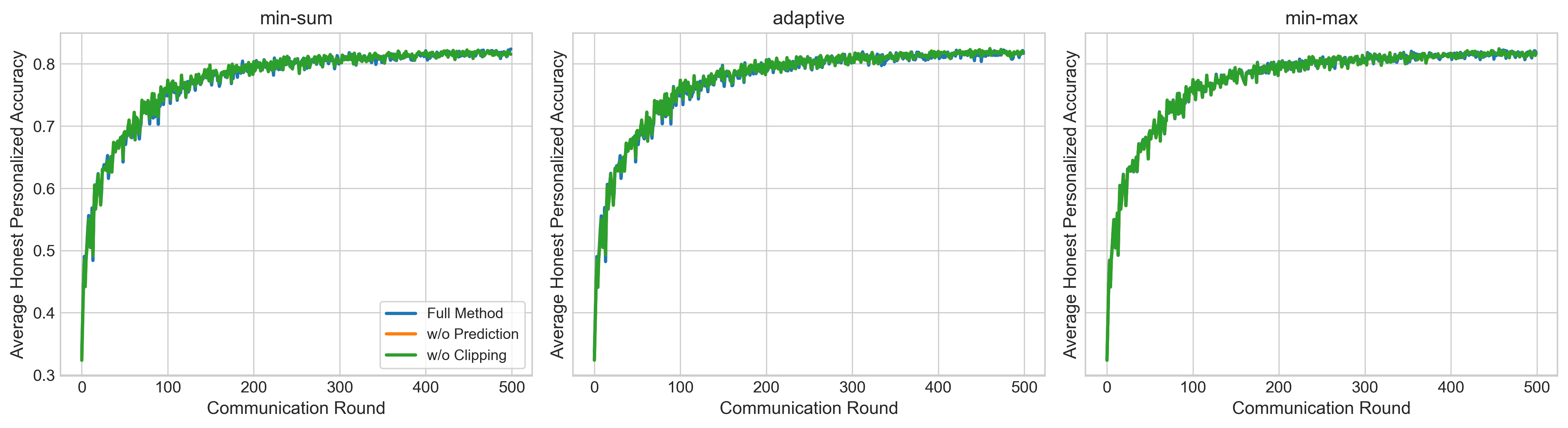}
\end{figure*}

To further understand the contribution of the key design components
in the proposed method, we conduct an ablation study by comparing
the full method with two reduced variants, namely w/o Prediction and
w/o Clipping. The former removes the neighborhood direction prediction
module, while the latter removes the residual clipping mechanism after
prediction-based correction. Through this study, we aim to examine
whether these two components are both beneficial to robustness under
Byzantine attacks.

Table \ref{tab:Ablation} summarizes the final results of different
ablated variants under representative Byzantine attacks, and Fig.\ref{fig:ablation}
further presents the corresponding final honest personalized accuracy.
Several observations can be made. First, the full method achieves
the most favorable overall performance across the considered attacks,
indicating that the complete prediction-and-correction method provides
the strongest robustness among the compared variants. Second, the
advantage of the full method is particularly clear under the min-sum
attack, where removing either prediction or clipping leads to a more
noticeable performance drop. This suggests that both components contribute
to suppressing harmful neighborhood deviations under stronger optimization-based
attacks. Third, under the adaptive and min-max attacks, the two reduced
variants perform similarly and remain slightly below the full method.

Fig. \ref{fig:ablation-1} shows the training curves of average honest
personalized accuracy for the three variants under min-sum, adaptive,
and min-max attacks. We observe that all variants remain stable during
training, but the full method generally maintains a more favorable
convergence trajectory, especially under min-sum. This result is consistent
with the final-performance comparison and further suggests that the
complete method can better preserve the quality of honest clients'
personalized models throughout training.

\begin{table}[tbh]
\caption{Sensitivity of the proposed method to different values of \textgreek{λ}
under representative Byzantine attacks.\label{tab:Sensitivity}}

\centering{}%
\begin{tabular}{|c|c|c|c|}
\hline 
attack\_type & $\lambda$ & Honest & All\tabularnewline
\hline 
\hline 
\multirow{4}{*}{min\_sum} & 0 & 0.8144 & 0.8185\tabularnewline
\cline{2-4}
 & 0.005 & 0.8165 & 0.8179\tabularnewline
\cline{2-4}
 & 0.01 & 0.8236 & 0.8231\tabularnewline
\cline{2-4}
 & 0.05 & 0.8229 & 0.824\tabularnewline
\hline 
\multirow{4}{*}{sign\_flip} & 0 & 0.8144 & 0.8185\tabularnewline
\cline{2-4}
 & 0.005 & 0.8172 & 0.8193\tabularnewline
\cline{2-4}
 & 0.01 & 0.8172 & 0.8195\tabularnewline
\cline{2-4}
 & 0.05 & 0.8145 & 0.8166\tabularnewline
\hline 
\end{tabular}
\end{table}

\begin{figure}[tbh]
\caption{Sensitivity of the proposed method to neighborhood collaboration strength
under representative Byzantine attacks.\label{fig:Sensitivity}}

\centering{}\includegraphics[scale=0.45]{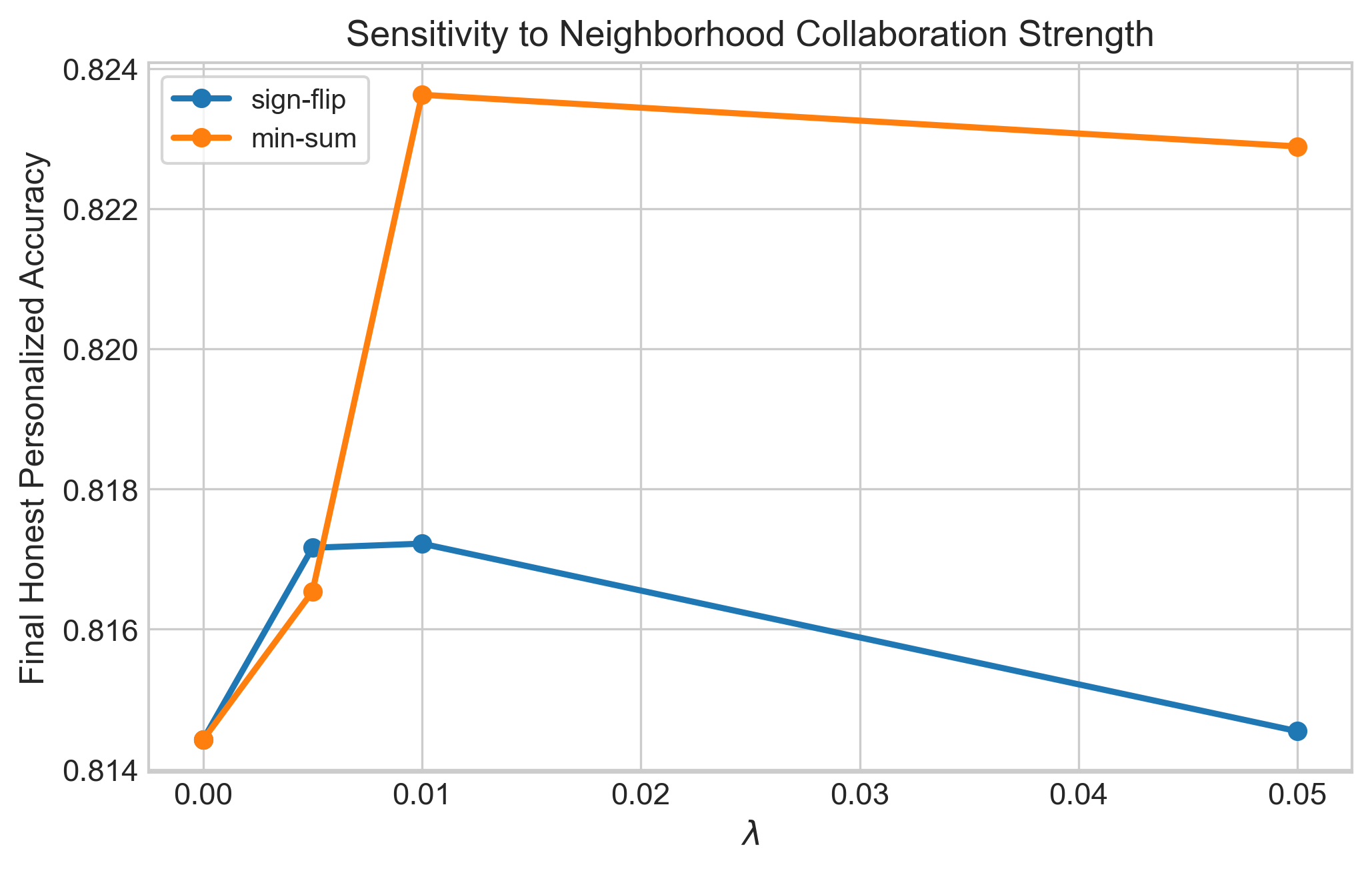}
\end{figure}

We further investigate the sensitivity of the proposed method to the
neighborhood collaboration strength controlled by $\lambda$. This
parameter determines the influence of the corrected neighborhood signal
on the final personalized update, and therefore plays an important
role in balancing local training and neighborhood collaboration under
Byzantine attacks.

Table \ref{tab:Sensitivity} and Fig. \ref{tab:Sensitivity} show
the final performance of the proposed method under different values
of $\lambda$ for two representative attack settings, namely sign-flip
and min-sum. Several observations can be made. First, when $\lambda=0$,
the method reduces to a purely local update without neighborhood correction,
and the resulting robustness is relatively limited. As $\lambda$
increases from 0 to a small or moderate range, the final average honest
personalized accuracy improves, indicating that incorporating neighborhood
information in a controlled manner is beneficial to robustness. Second,
the best performance is achieved at a moderate value of $\lambda$,
rather than at the largest tested value. In particular, for both sign-flip
and min-sum attacks, the performance improves when $\lambda$ increases
from 0 to 0.01, while the gain becomes marginal or slightly decreases
when $\lambda$ is further increased to 0.05.

Overall, these results indicate that the proposed method is not overly
sensitive to $\lambda$ within a moderate range, while an appropriate
neighborhood collaboration strength is important for achieving the
best robustness. In our experiments, $\lambda=0.01$ provides a favorable
trade-off between utilizing neighborhood information and avoiding
excessive correction.

\section{Conclusions}

In this paper, we proposed R-DPFL, a robust decentralized personalized
federated learning method for Byzantine environments. The key motivation
is that personalization and robustness can conflict in decentralized
learning: honest clients may naturally generate different updates
due to heterogeneous data and local objectives, making malicious perturbations
harder to distinguish from normal personalized differences. Therefore,
instead of forcing neighboring clients to learn the same model, our
R-DPFL keeps each client model personalized and uses neighbor information
only as a controlled correction to the local update. Specifically,
each client first constructs a local descent direction from its private
data and exchanges this direction with its neighbors. The received
neighbor directions are robustly aggregated into an initial neighborhood
update. R-DPFL then predicts what this aggregated neighborhood update
should be based on historical local model changes and historical aggregated
neighborhood updates. Finally, the difference between the current
aggregated update and the predicted update is adaptively clipped before
being added to the local descent direction.

Theoretically, we established a bounded personalized stationarity
guarantee, showing that honest clients can maintain stable descent
on their own local objectives under Byzantine neighbor perturbations
without enforcing consensus among neighboring models. Empirically,
experiments under heterogeneous decentralized settings and multiple
Byzantine attacks demonstrated that R-DPFL achieves strong robustness
and stable personalized performance.

\section*{Acknowledgment}

This work is supported by the Science and Technology Development Fund
of Macao (FDCT) (Project $\#$0015/2023/RIA1), Queensland State Department
of Environment and Science under the Quantum Challenges 2032 Program
(Project \#Q2032001). The corresponding author is Hong Shen.

\bibliographystyle{IEEEtran}
\bibliography{reference}

\end{document}